\documentclass{article}

\usepackage{times}
\usepackage{graphicx} 
\usepackage{subfigure}
\usepackage{float}

\usepackage{epstopdf}
\DeclareGraphicsExtensions{.eps}

\usepackage{amsfonts}
\usepackage{amsmath}
\usepackage{amssymb}
\usepackage{amsthm}
\usepackage{authblk}

\usepackage[square,comma,numbers,sort&compress]{natbib}

\usepackage{algorithm}
\usepackage{algorithmic}

\newcommand{\distas}[1]{\mathbin{\overset{#1}{\kern\z@\sim}}}%
\newsavebox{\mybox}\newsavebox{\mysim}
\newcommand{\distras}[1]{%
  \savebox{\mybox}{\hbox{\kern3pt$\scriptstyle#1$\kern3pt}}%
  \savebox{\mysim}{\hbox{$\sim$}}%
  \mathbin{\overset{#1}{\kern\z@\resizebox{\wd\mybox}{\ht\mysim}{$\sim$}}}%
}
\makeatother
\usepackage{arydshln}

\usepackage{hyperref}

\newtheorem{defn}{\noindent Definition}
\newtheorem{lem}{\noindent Lemma}
\newtheorem{thm}{\noindent Theorem}
\newtheorem{rem}{\noindent Remark}
\newtheorem{cor}{\noindent Corollary}
\newtheorem{example}{\noindent Example}

\global\long\def\lowrank#1#2#3{\mathcal{{M}}_{#1\times#2}^{(#3)}}
\global\long\def\affine{\mathcal{{A}}}
\global\long\def\matrixspace#1#2{\mathbb{{R}}_{#1\times#2}}
\global\long\def\orthomatrixspace#1#2{\mathcal{O}_{#1 \times#2}}
\global\long\def\identity#1{I_{#1}}
\global\long\def\matrixrow#1#2{#1_{#2\bullet}}
\global\long\def\matrixcolumn#1#2{#1_{\bullet#2}}
\global\long\def\svdtop#1#2{#1_{(#2)}}
\global\long\def\br{B^{(R)}}
\global\long\def\bc{B^{(C)}}
\global\long\def\brzero{B^{(R,0)}}
\global\long\def\bczero{B^{(C,0)}}
\global\long\def\ar{A^{(R)}}
\global\long\def\ac{A^{(C)}}
\global\long\def\zr{Z^{(R)}}
\global\long\def\zc{Z^{(C)}}
\global\long\def\pr{p^{(R)}}
\global\long\def\pc{p^{(C)}}
\global\long\def\kr{k^{(R)}}
\global\long\def\kc{k^{(C)}}
\global\long\def\aur{A_{\hat{U}}^{(R)}}
\global\long\def\avc{A_{V^{T}}^{(C)}}
\global\long\def\crow{c^{(R)}}
\global\long\def\ccol{c^{(C)}}

\global\long\def\psuedoinv#1{{#1}^{\dagger}}
\global\long\def\loss{\mathcal{\mathcal{{F}}}}
\global\long\def\llossll{\mathcal{\mathcal{{L}}}}
\global\long\def\alg{SVLS }  
\global\long\def\algp{SVLS_P }  

\global\long\def\algo{OptSpace }  
\global\long\def\iid{\overset{i.i.d.}{\sim}}
\global\long\def\snr{NR}

\global\long\def\measurements{d}
\global\long\def\RRMSE{RRMSE }

\newcommand{\ber}{\begin{eqnarray}}
\newcommand{\eer}{\end{eqnarray}}
\renewcommand{\baselinestretch}{1.2}

\newtheorem{claim}{\noindent Claim}

\newcommand{\be}{\begin{equation}}
\newcommand{\ee}{\end{equation}}
\newcommand{\bal}{\begin{align}}
\newcommand{\eal}{\end{align}}
\newcommand{\balnonum}{\begin{align*}}
\newcommand{\ealnonum}{\end{align*}}

\newcommand{\nin}{\noindent}

\begin{document}

\title{Low-Rank Matrix Recovery from Row-and-Column Affine Measurements}

\author{Avishai Wagner \hspace{3cm} Or Zuk  \\
avishai.wagner@mail.huji.ac.il \hspace{0.4cm} or.zuk@mail.huji.ac.il}
\affil{Dept. of Statistics, The Hebrew University of Jerusalem \\
            Mt. Scopus, Jerusalem, 91905, Israel}

\vskip 0.3in
\date{}
\maketitle

\begin{abstract}
We propose and study a row-and-column affine measurement scheme for
low-rank matrix recovery. Each measurement is a linear combination
of elements in one row or one column of a matrix $X$. This setting
arises naturally in applications from different domains. However,
current algorithms developed for standard matrix recovery
problems do not perform well in our case, hence the need for developing
new algorithms and theory for our problem. We propose a simple algorithm
for the problem based on Singular Value Decomposition ($SVD$) and least-squares
($LS$), which we term \alg. We prove that (a simplified version of)
our algorithm can recover $X$ exactly with the minimum possible number
of measurements in the noiseless case. In the general noisy case,
we prove performance guarantees on the reconstruction accuracy under
the Frobenius norm. In simulations, our row-and-column design and
\alg algorithm show improved speed, and comparable and in some cases better accuracy compared
to standard measurements designs and algorithms.
Our theoretical and experimental results suggest that the proposed
row-and-column affine measurements scheme, together with our recovery
algorithm, may provide a powerful framework for affine matrix reconstruction.

\end{abstract}

\nin \textbf{Keywords:} low-rank matrix recovery, row and column measurements, matrix completion, singular value decomposition

\section{Introduction}

In the low-rank affine matrix recovery problem, an unknown matrix
$X\in\mathbb{R}_{n_{1}\times n_{2}}$ with $rank(X)=r$ is measured
indirectly via an affine transformation $\mathcal{{A}}:\mathbb{R}_{n_{1}\times n_{2}}\rightarrow\mathbb{R}^{\measurements}$
and possibly with additive (typically Gaussian) noise $z\in\mathbb{R}^{d}$.
Our goal is to recover $X$ from the vector of noisy measurements
$\mathbf{b}=\mathcal{{A}}(X)+z$. The problem has found numerous applications
throughout science and engineering, in different fields such as collaborative
filtering \cite{koren2009matrix}, face recognition \cite{basri2003lambertian},
quantum state tomography \cite{gross2010quantum} and computational
biology \cite{chi2013genotype}. The problem has been studied mathematically
quite extensively in the last few years. Most attention thus far has
been given to two particular ensembles of random transformations $\affine$:
(i) the Matrix Completion (MC) setting, in which each element of $\mathcal{{A}}(X)$
is a single entry of the matrix where the subset of the observed measurements is
sampled uniformly at random \cite{candes2009exact,Candes2010,candes2010power,keshavan2009matrix,keshavan2010matrix,recht2011simpler}
(ii) Gaussian-Ensemble (GE) affine-matrix-recovery, in which each
element of $\mathcal{{A}}(X)$ is a weighted sum of all elements of
$X$ with i.i.d. Gaussian weights \cite{candes2011tight,recht2010guaranteed}.
Remarkably, although the recovery problem is in general NP-hard, when
$r\ll min(n_{1},n_{2})$ and under certain conditions on the matrix
$X$ or the measurement operator $\mathcal{{A}}$,
one can recover $X$ from $d\ll n_{1}n_{2}$ measurements with high
probability and using efficient algorithms \cite{candes2009exact,recht2010guaranteed,candes2010power,recht2011simpler}.
However, it is desirable to study the problem with other affine transformations
$\affine$ beyond the two ensembles mentioned above for the following
reasons: (i) In some applications we cannot control the measurements
operator $\mathcal{{A}}$, and different models for the measurements
may be needed to allow a realistic analysis of the problem (ii) When
we can control and design the measurement operator $\affine$,
other measurement operators may outperform the two ensembles mentioned above with respect
to optimizing different resources such as the number of measurements required, computation
time and storage. The main goal of this paper is to present and study
a different set of affine transformations, which we term row-and-column
affine measurements. This setting may arise naturally in many applications,
since it is often natural and possibly cheap to measure a single row
or column of a matrix, or a linear combination of a few such rows
and columns. For example, (i) In collaborative filtering, we may wish
to recover a users-items preference matrix and have access to only
a subset of the users, but can observe their preference scores for
\textit{all} items (ii) When recovering a protein-RNA interactions
matrix in molecular biology, a single experiment may simultaneously
measure the interactions of a specific protein with all RNA molecules
\cite{chu2011genomic}.

In general, we can represent any affine transformation $\affine$ in matrix representation
$\affine(X)=Avec(X)$, where $vec(X)$ is a column vector obtained by stacking all columns of $X$ on top of each other.
In our row and column framework the measurement operator $\affine$
is represented differently using two matrices $\ar,\ac$ which multiply $X$ as a matrix
(rather than multiplying the vector $vec(X)$) from left and right, respectively.
We focus on two ensembles of $\ar,\ac$: (i) Matrix Completion from single Columns and Rows
(RCMC). Here we observe single matrix entries in similar to standard
matrix completion case, however the measured entries are not scattered
randomly along the matrix, but instead we sample a few rows
and a few columns, and measure \textit{all} entries in these rows
and columns. This ensemble is implemented by setting the rows (columns)
of $\ar$ $(\ac)$ as random vectors from the standard basis of $\mathbb{R}^{n_{1}}$  ($\mathbb{R}^{n_{2}}$). (ii)
Gaussian Row-and-Column (GRC) measurements. Here each set of measurements
is a weighted linear combination of the matrix's rows (or columns)
with the weights taken as i.i.d. Gaussians. This ensemble is implemented
by setting the entries of $\ar,\ac$ as i.i.d. Gaussian random variables.

The measurement operators $\affine$ in our RCMC and GRC models do not satisfy
standard requirements which hold for GE and MC. It is thus not surprising that algorithms
such as nuclear norm minimization \cite{recht2010guaranteed,candes2009exact}, which
succeed for the GE and MC models, fail in our case,
and different algorithms and theory are required.
However, the specific algebraic structure provided by the row-and-column
measurements, allows us to derive efficient and simple algorithms, and to
analyze their performance.
In addition, we provide extensive simulation results, which demonstrate the improved accuracy and speed
of our approach over existing measurement designs and algorithms. All of our algorithms and simulations are implemented
in a Matlab software package available at \\
\url{https://github.com/avishaiwa/SVLS}.

\subsection{Prior Work}

Before giving a detailed derivation and analysis of our design and
algorithms, we give an overview of existing designs and their
properties. We concentrate on two properties: (i) storage required
in order to represent the measurement operator, and (ii) measurement sparsity,
defined as the sum over all measurements of the number of matrix entries
participating in each measurement, that is $S(\affine)=||vec(A)||_0$.  The latter property may be related
to measurement costs, as well as to computational time.

In the Gaussian Ensemble model, the entries of the matrix $A$ in the matrix representation
$\affine(X)=Avec(X)$ are i.i.d. Gaussian random variables, $A_{ij} \sim N(0,1)$.
For this ensemble, one can recover uniquely a low
rank matrix $X$ with $O(r (n_{1}+n_{2}))$ noiseless measurements using
nuclear norm minimization \cite{recht2010guaranteed,candes2011tight}
or other methods such as Singular Value Projection (SVP) \cite{jain2010guaranteed},
which is optimal up to constants. Recovery in this model is robust
to noise, with only a small increase in number of measurements. The main disadvantage
of this model is that the design requires $O(dn_{1}n_{2})$ storage space
for ${A}$, which could be problematic for large matrices. Another
possible disadvantage of this method is that measurements are dense
- each measurement represents a linear combination of all $O(n_{1}n_{2})$
matrix entries, and the overall measurement sparsity of $\affine(X)$
is also $O(dn_{1}n_{2})$, which could be problematic for large
$n_{1},n_{2}$.

In the standard matrix completion problem \cite{candes2009exact}
we can recover $X$ with high probability from single entries chosen uniformly
at random using nuclear norm minimization \cite{cai2010singular,toh2010accelerated,candes2010power,ma2011fixed,recht2011simpler}
or using other methods such as $SVD$ and gradient descent \cite{keshavan2009matrix,keshavan2010matrix}.
This model has the lowest storage requirements ($O(d)$) and measurement
sparsity ($O(d)$). However, recovery guarantees for this model are
weaker: setting $n=\max(n_1,n_2)$, it is shown that $\Theta(nrlog(n))$ measurements are required to recover a rank $r$ matrix of size $n_1 \times n_2$ \cite{candes2010power}.
In addition, unique recovery from this number of measurements requires additional incoherence conditions on the matrix $X$,
and recovery of matrices which fail to satisfy such conditions (e.g. matrices with a few spikes) may require a
much larger number of measurements.

Recently a new design of rank one projections was proposed
\cite{cai2015rop}, where each measurement is of the form $\alpha^{T}X\beta$
and such that $\alpha\in\mathbb{R}^{n_{1}} , \beta\in\mathbb{R}^{n_{2}}$
have i.i.d standard Gaussian entries. It was proven that nuclear norm
minimization can recover $X$ with high probability in this design from $O(n_1r+n_2r)$ measurements. This is the first
model deviating from MC and GE we are aware of. This model is different
from our row-and-column model, as each measurement is obtained by
multiplying $X$ from both sides, whereas in our model each measurement
is obtained by multiplying $X$ from either left or right. Moreover,
in our model the measurements are not chosen independently from each
other but come in groups of size $n_{1}$ or $n_{2}$ (corresponding
to rows or columns $\ar,\ac$). An advantage of rank one projection
is that it leads to a significance reduction in
measurement storage needed for $\affine$ with overall $O(dn_{1}+dn_{2})$
storage space. However, each measurement is still dense and involve
all matrix elements, hence measurement sparsity is $O(dn_{1}n_{2})$.
In contrast, our $GRC$ model requires only $O(d)$ storage for $\affine$,
and every measurement depends only on $O(n)$ elements,
leading to a reduced overall time for all measurements $O(dn_{1}+dn_{2})$.
For RCMC, we need only $O(\frac{dlog(n)}{n})$ storage for $\affine$,
and measurement sparsity is $O(d)$.

\section{Preliminaries and Notations\label{sec:Preliminaries-and-Notations}}

We denote by $\matrixspace{n_{1}}{n_{2}}$ the space of matrices of
size $n_{1}\times n_{2}$, by $\orthomatrixspace{n_{1}}{n_{2}}$ the
space of matrices of size $n_{1}\times n_{2}$ with orthonormal columns,
and by $\lowrank{n_{1}}{n_{2}}r$ the space of matrices of size $n_{1}\times n_{2}$
and rank $\leqslant r$. We denote $n=\max(n_1,n_2)$.

We denote by $||\cdot||_{F}$ the matrix Frobenius norm, by $||\cdot||_{*}$
the nuclear norm, and by $||\cdot||_{2}$ the spectral norm. For a
vector, $||\cdot||$ denotes the standard $l_{2}$ norm.

For $X\in\matrixspace{n_{1}}{n_{2}}$we denote by $span(X)$ the subspace
of $\mathbb{R}^{n_{1}}$ spanned by the columns of $X$ and define
$P_{X}$ to be the orthogonal projection into $span(X)$.

For a matrix $X$ we denote by $\matrixrow Xi$ the $i$-th row, by
$\matrixcolumn Xj$ the $j$-th column and by $X_{ij}$ the $(i,j)$
element. For two sets of indices $I,J$, we denote by $X_{IJ}$ the
sub-matrix obtained by taking the rows with indices in $I$ and columns
with indices in $J$ of $X$. We denote by $[k]$ the set of indices
$1,..,k$. We denote by $vec(X)$ the (column) vector obtained by
stacking all the columns of $X$ on top of each other.

We use the notation $X\iid G$ to denote a random matrix $X$ with
i.i.d. entries $X_{ij}\sim G$.

For a rank-$r$ matrix $X\in\lowrank{n_{1}}{n_{2}}r$ let $X=U\Sigma V^{T}$
be the Singular Value Decomposition ($SVD$) where 
$U\in\orthomatrixspace{n_{1}}r,\thinspace V\in\orthomatrixspace r{n_{2}}$
and $\Sigma=diag(\sigma_{1}(X),...,\sigma_{r}(X))$ with 
$\sigma_{1}(X)\geq\sigma_{2}(X)\geq..\geq\sigma_{r}(X)>0$
the (non-zero) singular values of $X$ (we omit the zero singular
values and their corresponding vectors from the decomposition). For
a general matrix $X\in\matrixspace{n_{1}}{n_{2}}$ we denote by $\svdtop Xr$
the top-$r$ singular value decomposition of $X$, $\svdtop Xr=U_{\bullet[r]}\Sigma_{[r][r]}V_{\bullet[r]}^{T}$.

Our model assumes two affine transformations applied to $X$, representing
rows and columns, $\bczero=X\ac$ and $\brzero=\ar X,$ achieved by
multiplications with two matrices $\ar\in\matrixspace{k^{(R)}}{n_{1}}$
and $A^{(C)}\in\matrixspace{n_{2}}{k^{(C)}}$, respectively. We obtain noisy observations
of these transformations, $\br,\bc$ obtained by applying additive
noise:
\begin{equation}
\ar X+Z^{(R)}=B^{(R)}\,;\quad XA^{(C)}+Z^{(C)}=B^{(C)}\label{eq:noisy_case}
\end{equation}
 where the total number of measurements is $\measurements=k^{(R)}n_{1}+n_{2}k^{(C)}$,
and $Z^{(R)}\in\mathbb{R}_{n_{1}\times k^{(R)}}$,$Z^{(C)}\in\mathbb{R}_{k^{(C)}\times n_{2}}$
are two zero-mean noise matrices. Our goal is to recover $X$ from
the observed measurements $B^{(C)}$ and $B^{(R)}$. To achieve this
goal, we define the squared loss function
\begin{equation}
\mathcal{\loss}(X)=||\ar X-B^{(R)}||_{F}^{2}+||XA^{(C)}-B^{(C)}||_{F}^{2}\label{eq:losses_function}
\end{equation}
 and solve the least squares problem:
\begin{equation}
Minimize\:\loss(X)\: s.t.\: X\in\lowrank{n_{1}}{n_{2}}r.\label{eq:problem1}
\end{equation}
 If $Z^{(R)},Z^{(C)}\iid N(0,\tau^{2})$ , minimizing the loss function
in eq. (\ref{eq:losses_function}) is equivalent to maximizing the
log-likelihood of the data, giving a statistical motivation for the
above score. Problem (\ref{eq:problem1}) is non-convex due to the
non-convex rank constraint $rank(X)\leq r$.

Our problem is a specialization of the general affine matrix recovery
problem \cite{recht2010guaranteed}, in which a matrix is measured
using a general affine transformation $\affine$ with $\mathbf{b}=\affine(X)+z$.
We consider next and throughout the paper two specific \textit{random ensembles} of measurement matrices: 
\begin{enumerate}
\item \textbf{Row and Column Matrix Completion (RCMC):} In this ensemble
each row of $\ar$ and each column of $\ac$ is a vector of the standard
basis $e_{j}$ for some $j$ - thus each measurement $\br_{ij}$ or
$\bc_{ij}$is obtained from a single entry of $X$. We define a row-inclusion
probability $\pr$ and column inclusion probability $\pc$ such that
each row (column) of $X$ will be measured with probability
$\pr$ ($\pc$). More precisely, we define $r_{1},..,r_{n_{1}}$ i.i.d.
Bernoulli variables, $P(r_{i}=1)=\pr$, and include $e_{i}$ as a
row in $\ar$ if and only if $r_{i}=1$. Similarly, we define $c_{1}...c_{n_{2}}$
i.i.d. Bernoulli variables, $P(c_{i}=1)=\pc$, and include $e_{i}$
as a column in $\ac$ if and only if $c_{i}=1$. The expected number
of observed rows (columns) is $\kr=n_{1}\pr$ ($\kc=n_{2}\pc$). The
model is very similar to the possibly more natural model of picking
$\kr$ distinct rows and $\kc$ distinct columns at random for fixed
$\kr,\kc$, but allows for easier analysis.

\item \textbf{Gaussian Rows and Columns (GRC):} In this ensemble $\ar,\ac\iid N(0,1)$.
Each observation $\br_{ij}$ or $\bc_{ij}$ is obtained by a weighted
sum of a single row or column of $X$, with i.i.d. Gaussian weights.
\end{enumerate}

\subsection{Comparison to Standard Designs }

Our proposed rows-and-columns design differs from standard
designs appearing in the literature. It
is instructive to compare our GRC ensemble to the Gaussian Ensemble
(GE) \cite{candes2011tight}, with the matrix representation
$\affine(X)=Avec(X)$ where $A\in\matrixspace{\measurements}{n_{1}n_{2}}$ and $A\iid N(0,1)$. For the latter, the following $r$-Restricted
Isometry Property (RIP) can be used:
\begin{defn}
(r-RIP) Let $\affine:\matrixspace{n_{1}}{n_{2}}\rightarrow\mathbb{R}^{\measurements}$
be a linear map. For every integer $r$ with $1\le r\le min(n_{1},n_{2})$,
define the $r$-Restricted Isometry Constant to be the smallest number
$\epsilon_{r}$ such that

\begin{equation}
(1-\epsilon_{r})||X||_{F}\leq ||\affine(X)|| \leq(1+\epsilon_{r})||X||_{F}
\end{equation}
holds for all matrices $X$ of rank at most $r$.
\end{defn}

The GE model satisfies the $r$-RIP condition for $d = O(rn)$ with
high probability \cite{recht2010guaranteed}. Based on this property
it is known that nuclear norm minimization \cite{recht2010guaranteed,candes2011tight}
and other algorithms such as SVP \cite{jain2010guaranteed}
can recover $X$ with high probability. Unlike GE, in our GRC model
$\affine(X)$ doesn't satisfy the $r$-RIP, and nuclear
norm minimization fails. Instead, $\ar,\ac$ preserve matrix Frobenius
norm in high probability - a weaker property which holds for \textit{any} low-rank matrix.
(see Lemma \ref{lem:JL_Lemma} in the Appendix).

We next compare RCMC to the standard Matrix Completion model
\cite{candes2009exact}, in which single entries are chosen at random
to be observed. Unlike GE, for MC incoherence conditions on $X$ are
required in order to guarantee unique recovery of  $X$ \cite{candes2009exact} :
\begin{defn}
\label{def:-(Incoherence)}(Incoherence). Let $U$ be a subspace of
$\mathbb{R}^{n}$ of dimension $r$, and $P_{U}$ be the orthogonal
projection on $U$. Then the coherence of $U$ (with respect to the
standard basis $\{e_{i}\}$) is defined as
\begin{equation}
\text{\ensuremath{\mu}}(U)\text{\ensuremath{\equiv}}\frac{n}{r}max_{i}||P_{U}(e_{i})||^{2}.\label{eq:incoherence}
\end{equation}

\end{defn}
We say that a matrix $X\in\matrixspace{n_{1}}{n_{2}}$ is $\mu$-incoherent
if for the $SVD$ $X=U\Sigma V^{T}$ we have $max(\mu(U),\mu(V))\leq\mbox{\ensuremath{\mu}}$.

When $X$ is $\mu$-incoherent, and when known entries are sampled uniformly
at random from $X$, several algorithms \cite{keshavan2009matrix,cai2010singular,jain2010guaranteed}
succeed to recover $X$ with high probability. In particular, nuclear
norm minimization has gained popularity as a solver for the standard
MC problem because it provides recovery guarantees and a convenient representation
as a convex optimization problem with availability of many iterative
solvers for the problem. However, nuclear norm minimization fails
for the RCMC design, even when the matrix $X$ is incoherent, as shown
by the next example:

\begin{example}
Take $X\in\matrixspace nn$ for $\frac{n}{3}\in\mathbb{N}$
with $X_{ij}=1,\:\:\forall(i,j)\in[n]\times[n]$. Thus $||X||_{*}=n$.
Take $k^{(R)}=k^{(C)}=\frac{n}{3}$. Set all unknown
entries to $0.5$, giving a matrix $X_{0}$ of rank $2$
with $\sigma_{1}(X_{0})=\frac{(\sqrt{2}+1)n}{3}$, $\sigma_{2}(X_{0})=\frac{(\sqrt{2}-1)n}{3}$.
Therefore $||X_{0}||_{*}=\frac{n\sqrt{2}}{3}<||X||_{*}$ and nuclear
norm minimization fails to recover the correct $X$.
\end{example}

In Section \ref{sec:Our-method} we present our SVLS algorithm, which does not rely on nuclear-norm minimization.
In Section \ref{sec:Performance-Guarantees} we show that SVLS successfully approximates $X$ for the GRC ensemble.

\section{Algorithms for Recovery of $X$ \label{sec:Our-method}}

In this section we give an efficient algorithm which we call \alg
(Singular Value Least Squares). \alg is very easy to implement - for simplicity, we start with
Algorithm \ref{alg:noiseless_case} for the noiseless case and then present Algorithm \ref{alg:basis_noisy_case} (\alg\!) which is applicable
for the general (noisy) case.

\subsection{Noiseless Case}

In the noiseless case we reduce the optimization problem (\ref{eq:problem1})
to solving a system of linear equations \cite{candes2009exact}, and provide Algorithm \ref{alg:noiseless_case}, which often leads to a closed-form
estimator. We then give conditions under which with high probability,
the closed-form solution is unique and is equal to the true matrix $X$.
\begin{algorithm}[h]
Input: $A^{(R)},A^{(C)},B^{(R)},B^{(C)}$ and rank $r$
\begin{enumerate}
\item Compute a basis (of size $r$) to the column space of $B^{(C)}$ using Gaussian elimination,
represented as the columns of a matrix $\hat{U} \in \matrixspace {n_{1}}r $.
\item Solve the linear system $B_{\bullet j}^{(R)}=A^{(R)}\hat{U}Y_{\bullet j}$ for
each $j=1,..,n_2$ and write the solutions as a matrix $Y=Y_{\bullet1}...Y_{\bullet n_{2}}$.
\item Output $\hat{X}=\hat{U}Y$
\end{enumerate}
\protect\caption{\label{alg:noiseless_case}}
\end{algorithm}
 If $rank(A^{(R)}\hat{U})=r$ one can write the resulting estimator
$\hat{X}$ in closed-form as follows:
\begin{equation}
\hat{X}=\hat{U}Y=\hat{U}[\hat{U}^{T}A^{(R)^{T}}A^{(R)}\hat{U}]^{-1}\hat{U}^{T}A^{(R)^{T}}B^{(R)} \label{eq:algorithm1}
\end{equation}

Algorithm \ref{alg:noiseless_case} does not treat the row and column measurements symmetrically.
We can apply the same algorithm, but changing the role of rows and
columns. The resulting closed form solution is then:
\begin{equation}
\hat{X}=B^{(C)}A^{(C)}\hat{V}[\hat{V}^{T}A^{(C)}A^{(C)^{T}}\hat{V}]^{-1}\hat{V}^{T} \label{eq:algorithm1_columns}
\end{equation}
 for an orthogonal matrix $\hat{V}$ representing a basis for the
rows of $X$.
Since the algorithm uses Gaussian elimination steps for solving systems of linear equations,
it is crucial that we have exact noiseless measurements. Next, we
modify the algorithm to work also for noisy measurements.

\subsection{General (Noisy) Case}

In the noisy case we seek a matrix $X$ minimizing the loss $\loss$
in eq. (\ref{eq:losses_function}). The minimization problem is
non-convex and there is no known algorithm with optimality guarantees.
We propose Algorithm \ref{alg:basis_noisy_case} (\alg\!), which empirically returns a matrix
estimator $\hat{X}$ with a low value of the loss $\loss$. In addition, we prove in Section \ref{sec:Performance-Guarantees}
recovery guarantees on the performance of \alg\!.
\begin{algorithm}[H]
Input: $A^{(R)},A^{(C)},B^{(R)},B^{(C)}$ and rank $r$
\begin{enumerate}
\item Compute $\hat{U}$, the $r$ largest left singular vectors of $B^{(C)}$, ($\hat{U}$ is a basis for the columns space of $B^{(C)}_{(r)}$).
\item Find the least-squares solution
\be
\hat{Y}=argmin_{Y}\parallel B^{(R)}-A^{(R)}\hat{U}Y||_F.
\ee
If $rank(A^{(R)}\hat{U})=r$ we can write $\hat{Y}$ in closed form
as before:
\begin{equation}
\hat{Y}=[\hat{U}^{T}A^{(R)^{T}}A^{(R)}\hat{U}]^{-1}\hat{U}^{T}A^{(R)^{T}}B^{(R)}. \label{eq:close_formula_alg2}
\end{equation}
\item Compute the estimate $\hat{X}^{(R)}=\hat{U}\hat{Y}$.
\item Repeat steps 1-3, replacing the roles of columns and rows to get
an estimate $\hat{X}^{(C)}$.
\item Set $\hat{X} = argmin_{\hat{X}^{(R)},\hat{X}^{(C)}} \loss(X)$, for the loss $\loss(X)$ given in eq. (\ref{eq:losses_function}).
\end{enumerate}
\protect\caption{\label{alg:basis_noisy_case} \alg}
\end{algorithm}

\subsubsection{Gradient Descent}

The estimator  $\hat{X}$ returned by \alg may not minimize
the loss function $\loss$ in eq. (\ref{eq:losses_function}). We therefore perform
an additional gradient descent stage starting from $\hat{X}$ to achieve
an estimator with lower loss (while still possibly only a local minimum since the
problem is non-convex). \alg can be thus viewed as a fast method for providing a desirable
starting point for local-search algorithms. The details of the gradient descent are
given in the Appendix, Section \ref{sec:gradient_descent_appendix}.

\subsection{Estimation of Unknown Rank}
\label{sec:estimate_unknown_rank}
In real life problems, one doesn't know the true rank of a matrix
and should estimate it from data. Our rows-and-columns sampling design
is particularly suitable for rank estimation since $rank(\bczero)=rank(\brzero)=rank(X)$
with high probability when enough rows and columns are sampled. In
the noiseless case we can estimate $rank(X)$ by $\hat{r}$=$rank(\bczero)$ or $rank(\brzero)$.

For the noisy case we estimate $rank(X)$ from $B^{(C)},B^{(R)}$. We use the popular elbow method
to estimate $rank(B^{(C)})$ in the following way
\begin{equation}
\hat{r}^{(C)}=argmax_{i\in[k^{(C)}-1]}\left(\frac{\sigma_{i}(\bc)}{\sigma_{i+1}(\bc)}\right) \label{eq:rank_estimation}
\end{equation}
We compute similarly $\hat{r}^{(R)}$ from $B^{(R)}$ and take the
average as our rank estimator, $\hat{r}=round\left(\frac{\hat{r}^{(C)}+\hat{r}^{(C)}}{2}\right)$.
We demonstrate the performance of our rank estimation using simulations in the Appendix, Section \ref{sec:Rank estimation}.

Modern methods for rank estimation from singular values \cite{Gavish2013} can be similarly applied to $B^{(R)},B^{(C)}$
and may yield more accurate rank estimates. After we estimate the rank, we can plug-in $\hat{r}$ as the rank
parameter in the \alg algorithm and recover $X$.  
\subsection{Low Rank Approximation \label{sec:Low_rank_aprox}}
In the low rank matrix approximation problem, the goal is to approximate a (possibly full rank) matrix $X$
by the closest (in Frobenius norm) rank-$r$ matrix $X_{(r)}$.
By the  Eckart-Young Theorem \cite{eckart1936approximation}, this problem has a closed-form solution which is the truncated $SVD$ of $X$.
$SVD$ is a powerful tool in affine matrix recovery and different algorithms such as SVT, \algo, SVP and others apply $SVD$.
In \cite{halko2011finding} the authors try to find a low rank approximation to $X$ using measurements $X{\ac}=\bc$ and ${\ar}X=\br$.
For large $n_1,n_2$ they give a single-pass algorithm which computes $X_{(r)}$ using only $\bc$ and $\br$.
We bring their algorithm in the Appendix, Section \ref{sec:Low Rank matrix}.
The main difference between the above formulation and our problem in eq. (\ref{eq:problem1}) is the rank estimation.
In \cite{halko2011finding} it is assumed that $k^{(R)}=k^{(C)}=k$ and one estimates $X_{(k)}$ instead of a rank-$r$ matrix which can lead to poor performance if $r \ll k$.
We adjusted the algorithm presented in \cite{halko2011finding} to our problem and give a new estimator which is a combination of \alg and \cite{halko2011finding}'s method,
replacing $\hat{X}^{(R)}$ and $\hat{X}^{(C)}$ in steps 3,4 of \alg by:

\begin{equation} \label{Tropp_estimation}
\quad \hat{X}^{(R)}_{P}=\hat{X}^{(R)}\hat{V}\hat{V}^T \:,
\quad \hat{X}^{(C)}_{P}=\hat{U}\hat{U}^T\hat{X}^{(C)}.
\end{equation}

Here $\hat{V}$ is the $r$ largest right singular vectors of $\br$ and $\hat{U}$ is the $r$ largest left singular vectors of $\bc$. We call this new estimator $\algp$.
Simulations show almost identical and in some cases slightly better performance of this modified algorithm compared to \alg (see Appendix, Section \ref{sec:Low Rank matrix}).
This modified estimator is however difficult to analyze rigorously, and therefore we present throughout the paper our results for the \alg estimator.

\section{Performance Guarantees\label{sec:Performance-Guarantees}}

In this section we give guarantees on the accuracy of the estimator $\hat{X}$ returned
by \alg\!. Our guarantees are probabilistic, with respect to randomizing
the design matrices $\ar,\ac$. For the noiseless case we give conditions
which are close to optimal for exact recovery.

\subsection{Noiseless Case}
\label{sec:noiseless_case}

A rank $r$ matrix of size $n_{1}\times n_{2}$ has $r(n_{1}+n_{2}-r)$
degrees of freedom, and can therefore not be uniquely recovered by fewer measurements.
Setting $k^{(R)}=k^{(C)}=r$ gives precisely this minimal number of
measurements. We next show that this number suffices, with probability
$1$, to guarantee accurate recovery of $X$ in the GRC model. In
the RCMC model the number of measurements is increased by a logarithmic factor in $n$
and we need an additional incoherence assumption on $X$ in order to guarantee accurate
recovery with high probability. We first present two Lemmas which will be useful.
Their proofs are given in the Appendix, Section \ref{proof:hatX_is_X}.
\begin{lem} \label{lem:uniqness solution affine} Let $X_{1},X_{2}\in\lowrank{n_{1}}{n_{2}}r$
and $A^{(R)}\in\matrixspace{k^{(R)}}{n_{1}},A^{(C)}\in\matrixspace{n_{2}}{k^{(C)}}$
such that $rank(A^{(R)}X_{1})=rank(X_{1}A^{(C)})=r$. If $A^{(R)}X_{1}=A^{(R)}X_{2}$
and $X_{1}A^{(C)}=X_{2}A^{(C)}$ then $X_{1}=X_{2}$.
\end{lem}

\begin{lem}
\label{lem:hatX_is_X} Let $X\in\lowrank{n_{1}}{n_{2}}r$ and $A^{(R)}\in\matrixspace{k^{(R)}}{n_{1}},A^{(C)}\in\matrixspace{n_{2}}{k^{(C)}}$
such that $rank(\ar X)=rank(X\ac)=r$. For Algorithm \ref{alg:noiseless_case} with inputs $A^{(R)},A^{(C)},B^{(R,0)},$ $B^{(C,0)}$
and $r$ the output $\hat{X}$ satisfies
\begin{equation}
A^{(R)}X=A^{(R)}\hat{X},\, XA^{(C)}=\hat{X}A^{(C)}
\end{equation}
\end{lem}

\subsubsection{Exact Recovery for GRC}

For the noiseless case, we can recover $X$ with the minimal number of measurements, as shown in Theorem \ref{cor:uniqness_GRC} (proof given in the Appendix, Section \ref{proof:uniqness_GRC}):
\begin{thm}
\label{cor:uniqness_GRC} Let $\hat{X}$ be the output of Algorithm \ref{alg:noiseless_case} in the GRC model with $Z^{(C)},Z^{(R)}=0$
and $k^{(R)},k^{(C)}\geq r$. Then $P(\hat{X}=X)=1$.
\end{thm}

\subsubsection{Exact Recovery for RCMC}

In the RCMC model, rows and columns of $X$ are sampled with replacement.
Since the same row can be sampled over and over, we cannot guarantee uniqueness of solution, as was the case for the
GRC model, but rather wish to prove that exact recovery of $X$ is possible with high probability. We assume
the Bernoulli rows and columns model as described in Section \ref{sec:Preliminaries-and-Notations}
and assume for simplicity that $\kr=\kc=k$.

\begin{thm}
\label{theorem:Uniqness_MC} Let $X=U\Sigma V^{T}$ be the $SVD$ of
$X \in \matrixspace {n_1}{n_2}$, and $\max(\mu(U),\mu(V))<\mu$. Take $A^{(R)}$
and $A^{(C)}$ as in the RCMC model without noise and probabilities
$\pr=\frac{k}{n_1}$ and $\pc=\frac{k}{n_2}$. Let $\beta>1$ such that $C_{R}\sqrt{\frac{\beta log(n)r\mu}{k}}<1$
where $C_{R}$ is uniform constant and let $\hat{X}$
be the output of Algorithm \ref{alg:noiseless_case}. Then $P\left(\hat{X}=X\right)>1-6min(n_1,n_2)^{-\beta}$.
\end{thm}

The proof of Theorem \ref{theorem:Uniqness_MC} is in the Appendix, Section \ref{proof:Uniqness_MC}.

\begin{rem}
Both row and column measurements are need in order
to guarantee unique recovery. If, for example, we observe only rows then even
with $n-1$ observed rows and rank $r=1$ we can only determine the
unobserved row up to a constant, and thus cannot recover $X$ uniquely.
\end{rem}

\subsection{General (Noisy) Case}

In the noisy case we cannot guarantee exact recovery of $X$, and our goal is to
minimize the error $||X-\hat{X}||_{F}$ for $\hat{X}$
the output of \alg\!. Here, we give bounds on the error for the GRC model. For simplicity, we show the result for $k^{(R)}=k^{(C)}=k$.

We focus on the high dimensional case $k\leq n$, where the number of
measurements is low. In this case our bound is similar to the bound
of the Gaussian Ensemble (GE). In \cite{candes2011tight} it is shown for GE that
$||X-\hat{X}||_{F}<C_{G}\sqrt{\frac{nr\tau^{2}}{\measurements}}$
holds with high probability for some constant $C_{G}$. We next give an analogous result
for our GRC model (proof in the Appendix, Section \ref{sec:proof_theorem_noisy_GRC}). %

{}

\begin{thm}
\label{thm:the_noisy_case} Let $\ar$ and $\ac$ with $k\geq max(4r,40)$
be as in the GRC model with noise matrices $\zr,\zc$. Let $\hat{X}$
be the output of SVLS. Then there exist constants $c,\crow,\ccol$ such that with probability $>1-5e^{-ck}$:

\begin{equation}
||X-\hat{X}||_{F} \leq \sqrt{\frac{r}{k}}  \Big[ \ccol ||\zc||_{2}+\crow ||\zr||_{2} \Big].
\label{eq:theorem_GRC_noisy}
\end{equation}
\end{thm}

Theorem \ref{thm:the_noisy_case} applies for any $\zc$ and $\zr$.
If $k\leq n$ and $\zr,\zc\iid N(0,\tau^{2})$, then from eq. (\ref{eq:matrix_singular_values})
we get $max(||\zr||_{2},||\zc||_{2}) \leq 4\tau \sqrt{n}$ with probability
$1-e^{-2n}$. We therefore get the next Corollary for i.i.d. additive Gaussian noise:

\begin{cor}
Let $\ar$, $\ac$ as in the GRC with $n\geq k\geq max(4r,40)$, model
and $\zr,\zc\iid N(0,\tau^{2})$. Then there exist constants $c,C_{GRC}$ such that:
\begin{equation}
P\Big(||X-\hat{X}||_{F}\leq C_{GRC} \sqrt{\frac{\tau^{2}nr}{k}}\Big) > 1-5e^{-ck}-e^{-2n} .
\end{equation}
\end{cor}

\section{Simulations Results}

We studied the performance of our algorithm using simulations.
We measured the reconstruction accuracy using the Relative Root-Mean-Squared-Error
($\RRMSE$), defined as
\begin{equation}
\RRMSE \equiv \RRMSE(X,\hat{X}) \equiv ||X - \hat{X}||_{F}/||X||_{F}.
\end{equation}
 For simplicity, we concentrated on square matrices with $n_{1}=n_{2}=n$
and used an equal number of row and column measurements, $k^{(R)}=k^{(C)}=k$
. In all simulations we sampled a random rank-$r$ matrix $X=UV^{T}$
with $U,V\in\matrixspace nr$ , $U,V\iid N(0,\sigma^{2})$.

In all simulations we assumed that $rank(X)$ is unknown and estimated
using the elbow method in eq. (\ref{eq:rank_estimation}).

\subsection{Row-Column Matrix Completion (RCMC)}
\label{sec:simulations_RCMC}

In the noiseless case we compared our design to standard MC.
We compared the reconstruction rate (probability of exact
recovery of $X$ as function of the number of measurements $\measurements$)
for the RCMC design with \alg to the reconstruction rate
for the standard MC design with the \algo \cite{keshavan2010matrix}
and SVT\cite{cai2010singular} algorithms. To allow for numerical
errors, for each simulation yielding $X$ and $\hat{X}$ we defined recovery
as successful if their $\RRMSE$ was lower than $10^{-3}$, and for
each value of $\measurements$ recorded the percentage of simulations
for which recovery was successful. In Figure \ref{figure:phase_transition}
we show results for $n=150,r=3$ and $\sigma=1$. \alg recovers
$X$ with probability $1$ with the optimal number of measurements $\measurements=r(2n-r)=894$
 yielding $\frac{d}{n^{2}}\approx0.04$ while
MC with \algo and SVT need roughly $3$-fold and $8$-fold more
measurements, respectively, to guarantee exact recovery.
\begin{figure}[H]
\begin{centering}
\includegraphics[width=11.5cm,height=8cm]{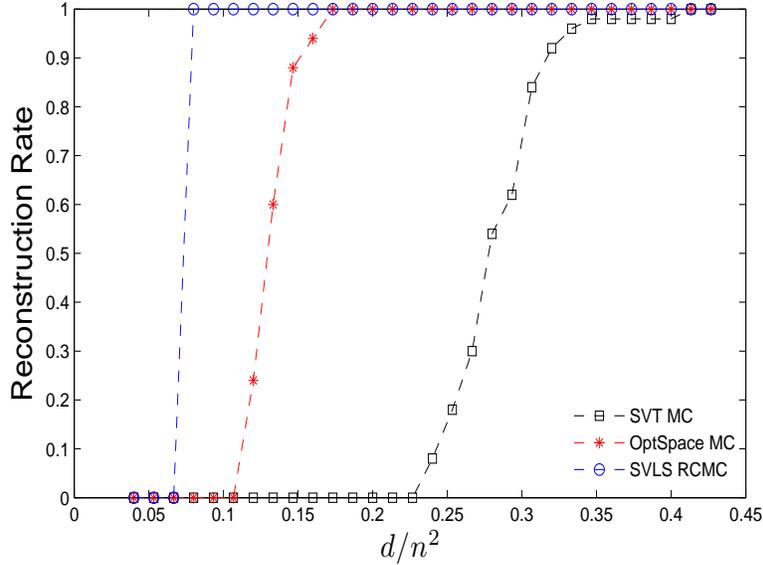}
\par\end{centering}

\protect\caption{Reconstruction rates for matrices with dimension $n=150$ and $r=3$
where $d$ is the number of known entries varied between $0$ to $8000$.
SVT and \algo are applied to the standard MC design and Algorithm
\ref{alg:noiseless_case} to RCMC. For each $\protect\measurements$
we sampled $50$ matrices and calculated the reconstruction rate as
described in the main text.}
\label{figure:phase_transition}
\end{figure}

The improvement in accuracy is not due to our design or our algorithm
alone, but due to their combination. We compared our method to \algo
and SVT for RCMC. We sampled a matrix $X$ with $n=100,r=3,\sigma=1$
and noise level $\tau^{2}=0.25^{2}$, and varied the number
of row and column measurements $k$. Figure \ref{fig:algorithm_performance}
shows that while the performance of \alg is very stable even for small $k$, the performance of \algo
varies, with multiple instances achieving poor accuracy, and SVT which minimizes the nuclear norm achieves
poor accuracy for all problem instances.
\begin{rem} The \algo algorithm has a trimming step
which delete dense columns. We omitted this step in the RCMC model
since it would delete all the known columns and rows and it's not
stable for this type of measurements, but it still get better result
than SVT.
\end{rem}

\begin{figure}[H]
\begin{centering}
\includegraphics[width=11cm,height=8cm]{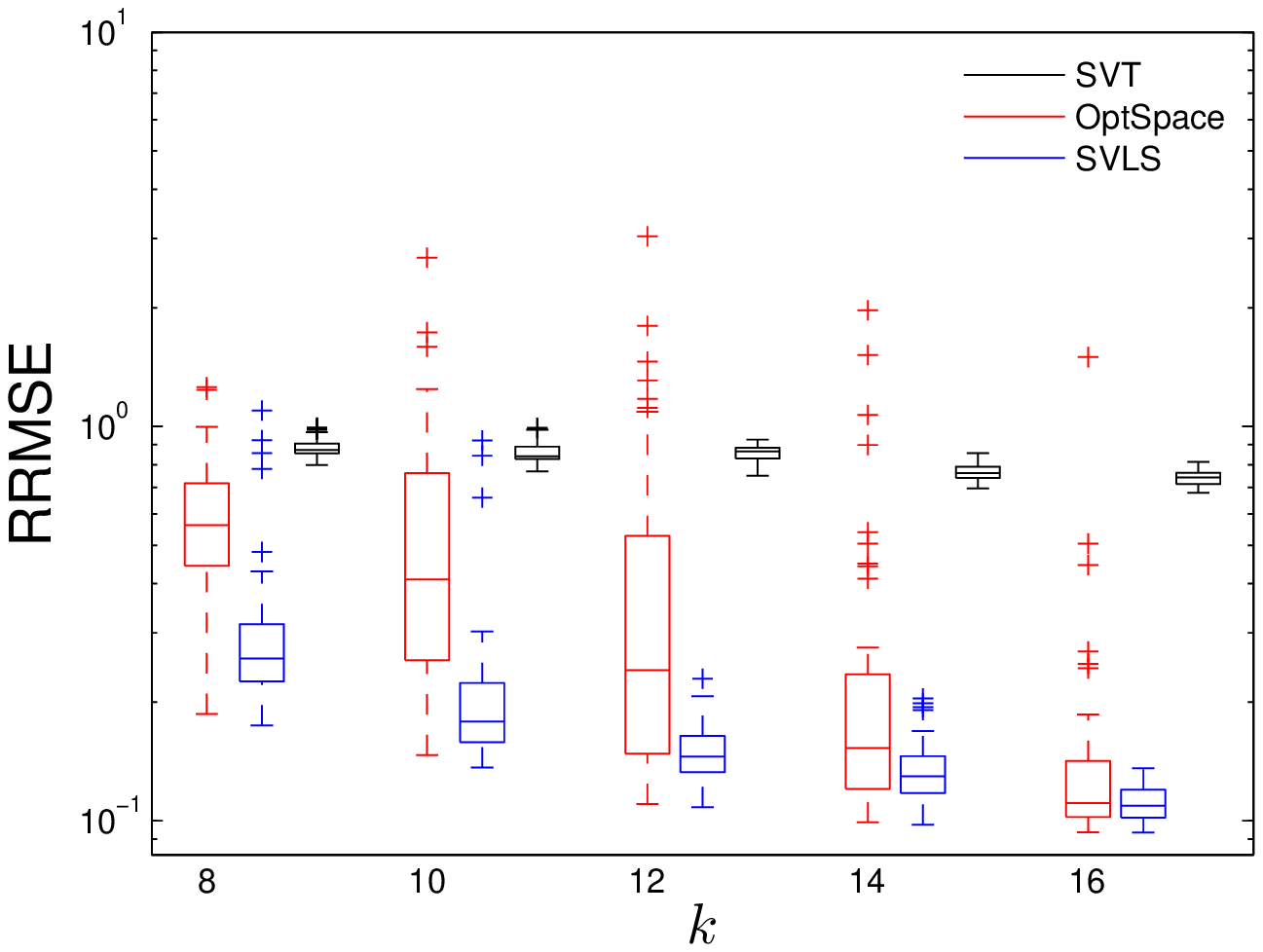}
\par\end{centering}

\protect\caption{Box-plots represent the distribution of $\protect\RRMSE$ as a function
of the number of column and row measurements $k$ over $50$ different
sampled matrices $X=UV^{T}$ with $U,V\protect\iid N(0,1)$ and $\protect\zr,\protect\zc\protect\iid N(0,0.25^{2})$.
\algo (red) fails to recover $X$ on many instances while \alg (blue) performs very well on all of them.
SVT (black) fails to recover $X$ for all instances. The trimming of dense rows and columns in
\algo was skipped, since such trimming in our settings may delete
all measurement information for low $k$. }
\label{fig:algorithm_performance}
\end{figure}

Next, we compared our RCMC to standard MC. We sampled $X$ as before
with $U,V\in\mathbb{R}_{1000\times r}$ with standard Gaussian distribution,
different rank and different noise ratio. The observations were corrupted
by additive Gaussian noise $Z$ with relative noise level $\snr \equiv ||Z||_{F}/||X||_{F}$.

Results, displayed in Table \ref{tab:accuracy-and-time}, show that \alg is significantly
faster than the other two algorithms. It is also more accurate than MC for small number
of measurements, and comparable to MC for large number of measurements.

\begin{table}
\centering
\small
\begin{tabular}{|c|c|c|c|c|c|}
  \hline
  \snr & $d$ & $r$ & \alg & \algo & SVT \\ \hline
  \!$10^{-2}\!\!\!$ & $120156$\! & \!$10$\! & \!$0.0063$ $(0.15)$\! & \!$0.004$ $(20.8)$\! & \!$0.0073$ $(18.7)$\! \\
  \!$10^{-1}\!\!\!$ & $120156$\! & \!$10$\! & \!$0.064$ $(0.15)$\! & \!$0.044$ $(21.7)$\! & \!$0.05$ $(11)$\! \\
  \!$1\!\!\!$ & $120156$\! & \!$10$\! & \!$0.612$ $(0.16)$\! & \!$0.49$ $(24.5)$\! & \!$0.51$ $(1)$\! \\
  \!$10^{-2}\!\!\!$ & $59100$\! & \!$20$\! & \!$0.029$ $(0.12)$\! & \!$0.97$ $(25.6)$\! & \!$0.76$ $(4.4)$\! \\
  \!$10^{-1}\!\!\!$ & $59100$\! & \!$20$\! & \!$0.3$ $(0.12)$\! & \!$0.98$ $(40.1)$\! & \!$0.86$ $(6.5)$\! \\
  \!$10^{-1}\!\!\!$ & $391600$\! & \!$50$\! & \!$0.081$ $(0.7)$\! & \!$0.05$ $(1200)$\! & \!$0.069$ $(13)$\! \\
  \!$1\!\!\!$ & $391600$\! & \!$50$\! & \!$0.72$ $(0.6)$\! & \!$0.61$ $(1300)$\! & \!$0.59$ $(5)$\! \\
  \hline
\end{tabular}
\caption{$\protect\RRMSE$ and time in seconds (in parenthesis) for \alg applied to RCMC, and
\algo and SVT applied to the standard MC. Results represent average of
$5$ different random matrices. \alg is faster than \algo and SVT by 1 to 3 orders of magnitudes, and shows comparable
or better $\protect\RRMSE$ in all cases. \label{tab:accuracy-and-time}}
\end{table}

Finally, we checked for RCMC and MC our performance only on unobserved entries,
to examine if $\RRMSE$ is optimistic due to overfitting to observed entries.
Results, shown in the Appendix, Section \ref{sec:test error}, indicate than no overfitting is observed.

\subsection{Gaussian Rows and Columns (GRC)}

We tested the performance of the GRC model with $A^{(R)},\ac\iid N(0,\frac{1}{n})$
and with $X=UV^{T}$ where $U,V \iid N(0,\frac{1}{\sqrt{r}})$. We compare our
results to the Gaussian Ensemble model (GE) where for each $n$, $\affine(X)$ was normalized to allow a fair comparison. In Figure \ref{fig:Relative-error-as-function-of-noise}
we take $n=100$ and $r=2$, and change the number of measurements $d=2nk$
(where $A^{(R)}\in\mathbb{R}_{k\times n}$ and $A^{(C)}\in\mathbb{R}_{n\times k}$).
We added Gaussian noise $Z^{(R)},Z^{(C)}$ with
different noise levels $\tau$. For all noise levels, the performance of GRC was better than the performance of GE.
The $\RRMSE$ error decays at a rate of $\sqrt{k}$. For GE
we used the APGL algorithm \cite{toh2010accelerated} for nuclear norm minimization.
\begin{figure}[H]
\begin{centering}
\includegraphics[width=11.5cm,height=7.5cm]{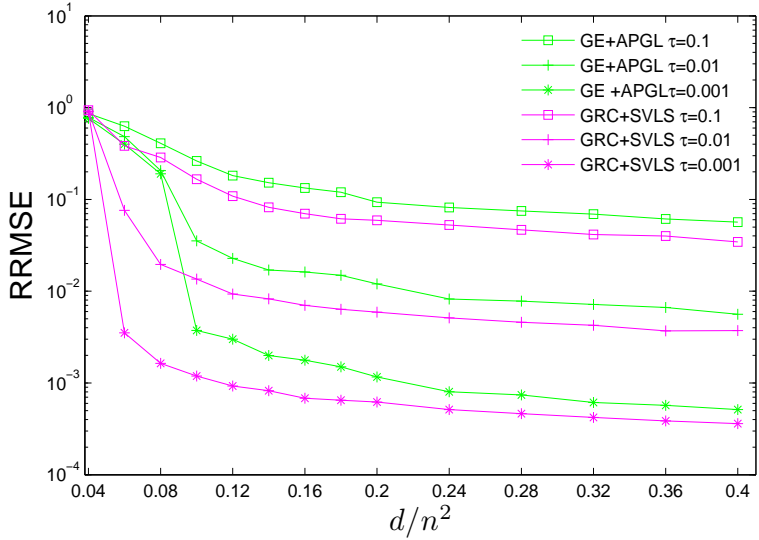}
\par\end{centering}

\protect\caption{\label{fig:Relative-error-as-function-of-noise} $\protect\RRMSE$
as function of $d$, the number of measurements, where we take $X\in\protect\lowrank{100}{100}2$,
$d$ is varied from $400$ to $4000$ and for different noise levels:
$\tau=0.1,0.01$ and $0.001$. For every point we simulated $5$ random
matrices and computed the average $\RRMSE$.}
\end{figure}

In the next tests we ran \alg for measurements with different noise levels. We take $n=1000$
and $k=100$ with different rank level every entry in $Z^{(C)},Z^{(R)}\iid N(0,\tau^{2})$
and different values of $\tau$. Results are shown in Figure \ref{fig:error_different_rates}.
The change in the relative error $\RRMSE$ is linear in $\tau$ while
the rate depends on $r$.

We next examined the behaviour of the $\RRMSE$ when $n
\to \infty$ and when $n,k,r\rightarrow\infty$ together, while the ratios $\frac{k}{n}$ and $\frac{\measurements}{r}$ are kept constant.
Results (shown in the Appendix, Section \ref{sec:appendix_simulations}) indicate that
when properly scaled, the $\RRMSE$ error is not sensitive to the value of $n$ and other parameters, in agreement
with Theorem \ref{thm:the_noisy_case}.

\section{Discussion}

We introduced a new measurements ensemble for low rank matrix recovery
where every measurements is an affine combination of a row or column
of $X$. We focused on two models: matrix completion from single columns
and rows (RCMC) and matrix recovery from Gaussian combination of columns
and rows (GRC). We proposed a fast algorithm for this ensemble.
For the RCMC model we proved that in the noiseless case our method
recovers $X$ with high probability and simulation results show that
the RCMC model outperforms the standard approach for matrix completion
in both speed and accuracy for models with small noise.

\begin{figure}[H]
\begin{centering}
\includegraphics[width=11cm,height=7.5cm]{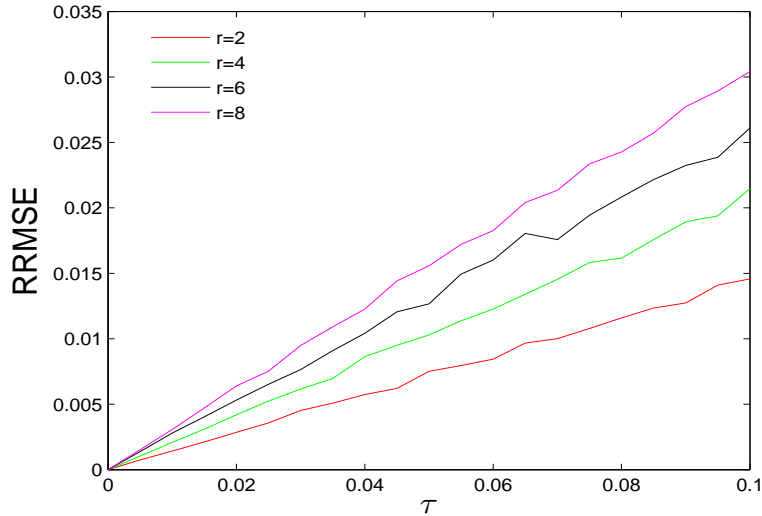}
\par\end{centering}

\protect\caption{\label{fig:error_different_rates} $\protect\RRMSE$ as a function of
noise level $\tau$ varied from $0$ to $0.1$, for matrices $X\in\protect\matrixspace{1000}{1000}$
of different ranks. For each curve we fitted a linear
regression line, with fitted slopes $0.145,0.208,0.25,0.3$ for $r=2,4,6,8$,
respectively. The slope is roughly proportional to $\sqrt{r}$ in
concordance with the error bound in Theorem \ref{thm:the_noisy_case}.
Further investigation of the relation using extensive simulations
is required in order to evaluate the dependency of the recovery error
in $r$ in a more precise manner.}
\end{figure}

For the GRC model we proved that our method recovers $X$ with the optimal
number of measurements in the noiseless case and gave an upper
bounds on the error for the noisy case. For RCMC, our simulations show that the
RCMC design may achieve comparable or favorable results, compared
to the standard MC design, especially for low noise level. Proving
recovery guarantees for this RCMC model is an interesting future challenge.

Our proposed measurement scheme is not restricted to recovery of
low-rank matrices. One can employ this measurement scheme and recover $X$ by minimizing other
matrix norms. This direction can lead to new algorithms that may improve matrix recovery for real
datasets.

\newpage

\bibliographystyle{plain} 
\bibliography{RC_arxiv_ICML2015}
\section*{\onecolumn}

\section{Appendix}

\subsection{Proofs for Noiseless GRC Case}

\subsection*{Proof of Lemma \ref{lem:uniqness solution affine}\label{proof:uniqness solution affine}}

\begin{proof}
First, $rank(X_{2}A^{(C)})=rank(X_{1}A^{(C)})=r$ and similarly 
$rank(A^{(R)}X_{2})= \allowbreak rank(A^{(R)}X_{1}) \allowbreak = \allowbreak r$.
Since $span(X_{1}A^{(C)}),span(X_{2}A^{(C)})$ are subspaces
of $span(X_{1}), \allowbreak span(X_{2})$ respectively, and $dim(span(X_{2})) \leq r$ we get
$span(X_{2})=span(X_{2}A^{(C)}) \allowbreak =span(X_{1}A^{(C)})=span(X_{1})$,
and we define $U\in\orthomatrixspace{n_{1}}r$ a basis for this subspace.
For $X_{1},X_{2}$ there are $Y_{1},Y_{2}\in\matrixspace r{n_{2}}$
such that $X_{1}=UY_{1},X_{2}=UY_{2}$. Therefore $A^{(R)}UY_{1}=A^{(R)}UY_{2}$.
Since $rank(A^{(R)}UY_{1})=r$ and $U\in\orthomatrixspace{n_{1}}r$
we get $rank(\ar U)=r$, hence the matrix $U^{T}A^{(R)^{T}}A^{(R)}U$
is invertible, which gives $Y_{1}=Y_{2}$, and therefore $X_{1}=UY_{1}=UY_{2}=X_{2}$.
\end{proof}

\subsection*{Proof of Lemma \ref{lem:hatX_is_X}\label{proof:hatX_is_X}}

\begin{proof}
$span(XA^{(C)})\subseteq span(X)$ and $rank(XA^{(C)})=rank(X)=r$,
therefore $span(X\ac)=span(X)$ and $\hat{U}$ from stage 1 in Algorithm \ref{alg:noiseless_case}
is a basis for $span(X)$. We can write $X=\hat{U} Y$ for some matrix $Y\in\mathbb{R}_{r\times n_{2}}$.
Since $rank(A^{(R)}\hat{U}Y)=rank(\hat{U})=r$, we have $rank(A^{(R)}\hat{U})=r$.
Thus eq. (\ref{eq:algorithm1}) gives $\hat{X}$ in closed form and we get:
\begin{align}
\ar\hat{X}=\ar\hat{U}[\hat{U}^{T}A^{(R)^{T}}\ar\hat{U}]^{-1}\hat{U}^{T}A^{(R)^{T}}\brzero= \nonumber\\
\ar\hat{U}[\hat{U}^{T}A^{(R)^{T}}\ar\hat{U}]^{-1}\hat{U}^{T}A^{(R)^{T}}\ar\hat{U}Y= \nonumber\\
\ar\hat{U}Y=\ar X . \label{eq:Row_equllity}
\end{align}

\vspace{-1.0cm}

\begin{align}
\hat{X}\ac=\hat{U}[\hat{U}^{T}A^{(R)^{T}}\ar\hat{U}]^{-1}\hat{U}^{T}A^{(R)^{T}}\ar X\ac= \nonumber\\
\hat{U}[\hat{U}^{T}A^{(R)^{T}}\ar\hat{U}]^{-1}\hat{U}^{T}A^{(R)^{T}}\ar\hat{U}Y\ac= \nonumber\\
\hat{U}Y\ac=X\ac . \label{eq:Col_equllity}
\end{align}

\vspace{-0.5cm}

\end{proof}


\begin{lem}
\label{lem:unitary_normal} Let $V\in\orthomatrixspace nr$ and $A^{(C)}\in\matrixspace nk$
be a random matrix $\ac\iid N(0,\sigma^{2})$. Then $V^{T}A^{(C)} \iid N(0,\sigma^{2})$.

\begin{proof}
For any two matrices $A\in\matrixspace{n_{1}}{n_{2}}$ and $B\in\matrixspace{m_{1}}{m_{2}}$
we define their Kronecker product as a matrix in $\matrixspace{n_{1}m_{1}}{n_{2}m_{2}}$:

\begin{equation}
A\otimes B=\left(
  \begin{array}{ccccc}
    a_{11}B & a_{12}B & . & . & a_{1n_2}B \\
    . & . & . & . & . \\
    . & . & . & . & . \\
    . & . & . & . & . \\
    a_{n_1 1}B & a_{n_1 2}B & . & . & a_{n_1 n_2}B \\
  \end{array} \label{eq:Kronecker_product}
\right)
\end{equation}

Now, we have $vec(V^{T}\ac)=(\identity n\otimes V^{T})vec(\ac)$ and since $vec(\ac)\sim N(0,\sigma \identity n)$
the vector $(\identity n\otimes V^{T})vec(\ac)$ is also a multivariate Gaussian vector with zero mean and covariance matrix:

\begin{align}
COV \Big( V^{T}A^{(C)} \Big) = COV \Big( (\identity n\otimes V^{T})vec(\ac) \Big) = \nonumber \\
(\identity n\otimes V^{T})COV\left(vec(A^{(C)})\right)(\identity n\otimes V^{T})^{T}= \nonumber \\
\sigma^{2}(\identity n\otimes V^{T})(\identity n\otimes V^{T})^{T}=\sigma^{2}\identity r\otimes\identity n = \sigma^2 \identity {nr}. \label{eq:VA_Kronecker}
\end{align}

\end{proof}
\end{lem}

\subsection*{Proof of Theorem \ref{cor:uniqness_GRC}\label{proof:uniqness_GRC}}

For the GRC model, Lemmas \ref{lem:uniqness solution affine},\ref{lem:hatX_is_X} and \ref{lem:unitary_normal} can be used to prove exact recovery
of $X$ with the minimal possible number of measurements:

\begin{proof}
Let $U\Sigma V^{T}$ be the $SVD$ of $X$. From Lemma \ref{lem:unitary_normal} the elements of the matrix $V^T \ac$ have
a continuous Gaussian distribution and since the measure of low rank matrices is zero and $k^{(C)}\geq r$
we get that $P(rank(V^{T} \ac)=r)=1$. Since $\bc=U\Sigma V^{T} \ac$
we get $P(rank(\bc)=rank(U\Sigma V^{T} \ac)=r)=1$. In the same way $P(rank(\br)=r)=1$.
Combining Lemma \ref{lem:hatX_is_X} with Lemma \ref{lem:uniqness solution affine}
give us the required result.
\end{proof}

\subsection{Gradient Descent}\label{sec:gradient_descent_appendix}

The gradient descent stage is performed directly in the space of rank $r$ matrices, using the decomposition
$\hat{X}$=$WS$ where $W\in\matrixspace{n_{1}}r$ and $S\in\matrixspace r{n_{2}}$
and computing the gradient of the loss as a function of $W$ and $S$,

\begin{equation}
\llossll(W,S) = \loss(WS) = ||\ar WS-B^{(R)}||_{F}^{2} + ||WSA^{(C)}-B^{(C)}||_{F}^{2}. \label{eq:loss_function_W_S}
\end{equation}

We want to minimize eq. (\ref{eq:loss_function_W_S}) but the loss $\llossll$ isn't convex and
therefore gradient descent may fail to converge to a global optimum. We propose $\hat{X}$ (the output of \alg\!) as a starting
point which may be close enough to enable gradient descent to converge to the global
optimum, and in addition may accelerate convergence.

The gradient of $\llossll$ is (using the chain rule)
$$
\frac{\partial \llossll}{\partial W}= 2 \Big[ A^{(R)^{T}}(\ar WS-B^{(R)})S^{T}+(WSA^{(C)}-B^{(C)})A^{(C)^{T}}S^{T} \Big]
$$
\begin{equation}
\frac{\partial \llossll}{\partial S}= 2 \Big[ W^{T}A^{(R)^{T}}(\ar WS-B^{(R)})+W^{T}(WSA^{(C)}-B^{(C)})A^{(C)^{T}} \Big] \label{eq:-1}
\end{equation}

\subsection{Proofs for Noiseless RCMC Case}

We prove that if $U\in\orthomatrixspace {n_1}r$ is orthonormal then with
high probability we have $p^{-1}||U^{T}A^{(R)^{T}}A^{(R)}U-pI_{r}||_{2}<1$.
Because $U$ is orthonormal, this is equivalent to

\be
p{-1} || UU^{T}A^{(R)^{T}}A^{(R)}UU^{T}-pUU^{T} ||_{2} < 1 \Leftrightarrow
p^{-1} || P_{U}P_{A^{(R)^{T}}}P_{U}-p P_{U}||_{2}<1 \label{eq:subspace_distance}
\ee

where $P_{U}=UU^{T},P_{A^{(R)^{T}}}=A^{(R)^{T}}A^{(R)}$ and $\pr=p$.
We generalize Theorem 4.1 from \cite{candes2009exact}.
\begin{lem}
\label{lemma: uniqnes_MC} Suppose $A^{(R)}$ as in the RCMC
model with inclusion probability $p$, and $U\in\orthomatrixspace {n_1}r$ with
$\mu(U)=\frac{n_1}{r}max_{i}||P_{U}(e_{i})||^{2}=\mu$. Then there is
a numerical constant $C_{R}$ such that for all $\beta>1$, if $C_{R}\sqrt{\frac{\beta log(n_1)r\mu}{pn_1}}<1$
then:
\be
P\left(p^{-1}||P_{U}P_{A^{(R)^{T}}}P_{U}-pP_{U}||_{2}<C_{R}\sqrt{\frac{\beta log(n_1)r\mu}{pn_1}}\right)
>1-3n_{1}^{-\beta}\label{eq:RCMC_performance_bound}
\ee

\end{lem}

The proof of Lemma \ref{lemma: uniqnes_MC} builds upon (yet generalizes)
the proof of Theorem 4.1 from \cite{candes2009exact}. We next present a few lemmas which are required for the proof of Lemma \ref{lemma: uniqnes_MC}.
We start with a lemma from \cite{candes2007sparsity}.
\begin{lem}
\label{lem:expection_bound-1} If $y_{i}$ is a family of vectors in
$\mathbb{R}^{d}$ and ${r_{i}}$ is a sequence of i.i.d. Bernoulli random variables
with $P(r_{i}=1)=p$, then
\begin{equation}
E \big(p^{-1} ||\Sigma_{i}(r_{i}-p)y_{i}\otimes y_{i}|| \big) < C\sqrt{\frac{log(d)}{p}}max_{i}||y_{i}||
\end{equation}
for some numerical constant $C$ provided that the right hand side is less
than $1$.
\end{lem}
We next use a result from large deviations theory \cite{talagrand1996new}:
\begin{thm}
\label{theorem: talngrad_inequality-1} Let $Y_{1}...Y_{n}$ be a
sequence of independent random variables taking values in a Banach
space and define
\begin{equation}
Z=sup_{f\in F}\sum_{i=1}^{n}f(Y_{i})
\end{equation}
where $F$ is a real countable set of functions such that if $f\in F$
then $-f\in F$.

Assume that $|f|\leq B$ and $E(f(Y_{i}))=0$ for every $f\in F$
and $i\in[n]$. Then there exists a constant $C$ such that for every
$t\geq0$
\begin{equation}
P\big(|Z-E(Z)|\geq t\big)\leq3exp\left(\frac{-t}{CB}log(1+\frac{t}{\sigma+Br})\right)
\end{equation}
where $\sigma=sup_{f\in F}\sum_{i=1}^{n}E(f^{2}(Y_{i}))$.
\end{thm}
Theorem \ref{theorem: talngrad_inequality-1} is used in the proof of the next
lemma which is taken from Theorem 4.2 in \cite{candes2009exact}. We
bring here the lemma and proof in our notations for convenience.
\begin{lem}
\label{lem:bound_Z-1} Let $U\in\orthomatrixspace nr$ with incoherence constant $\mu$.
Let $r_i$ be i.i.d. Bernoulli random variables with $P(r_i=1)=p$
and let $Y_{i}=p^{-1}(r_{i}-p)P_{U}(e_{i})\otimes P_{U}(e_{i})$ for $i=1,..,n$.
Let $Y=\sum_{i=1}^{n}Y_{i}$\textup{ and $Z=||Y||_{2}$. Suppose $E(Z)\leq 1$.
Then for every $\lambda>0$ we have
\begin{equation}
P\Big(|Z-E(Z)|\geq\lambda\sqrt{\frac{\mu rlog(n)}{pn}}\Big)\leq3exp\Big(-\gamma min(\lambda^{2}log(n),\lambda\sqrt{\frac{pnlog(n)}{\mu r}})\Big)
\end{equation}
 for some positive constant $\gamma$.}
 \end{lem}
\begin{proof}
We know that $Z=||Y||_{2}=sup_{f_{1},f_{2}} \langle f_{1},Yf_{2} \rangle =sup_{f_{1},f_{2}}\sum_{i=1}^{n}\langle f_{1,}Y_{i}f_{2} \rangle$,
where the supremum is taken over a countable set of unit vectors $f_{1},f_{2}\in F_{V}$.
Let $F$ be the set of all functions $f$ such that $f(Y)= \langle f_{1},Yf_{2} \rangle$
for some unit vectors $f_{1},f_{2}\in F_{V}$. For every $f\in F$
and $i\in[n]$ we have $E(f(Y_{i}))=0$. From the incoherence of $U$
we conclude that
\begin{equation}
|f(Y_{i})| = p^{-1} |r_{i}-p| \times | \langle f_{1},P_{U}(e_{i}) \rangle | \times | \langle P_{U}(e_{i}),f_{2} \rangle | \leq
p^{-1} ||P_{U}(e_{i})||^{2} \leq p^{-1}\frac{r}{n}\mu.
\end{equation}
 In addition
\[
E (f^{2}(Y_{i})) =p^{-1}(1-p) \langle f_{1},P_{U}(e_{i}) \rangle^{2} \langle P_{U}(e_{i}),f_{2} \rangle^{2} \leq
\]
\begin{equation}
p^{-1}||P_{U}(e_{i})||^{2}| \langle P_{U}(e_{i}),f_{2} \rangle^2 |\leq p^{-1}\frac{r}{n}\mu| \langle P_{U}(e_{i}),f_{2} \rangle|^{2}.\label{eq:f1f2}
\end{equation}
Since $\sum_{i=1}^{n}| \langle P_{U}(e_{i}),f_{2} \rangle|^{2}=\sum_{i=1}^{n}| \langle e_{i},P_{U}(f_{2}) \rangle |^{2}=||P_{U}(f_{2})||^{2} \leq 1$,
we get $\sum_{i=1}^{n} E (f^{2}(Y_{i})) \leq p^{-1}\frac{r}{n}\mu.$

We can take $B=2p^{-1}\frac{r}{n}\mu$ and $t=\lambda\sqrt{\frac{\mu rlog(n)}{pn}}$
and from Theorem \ref{theorem: talngrad_inequality-1}:
\begin{equation}
P(|Z-E(Z)|\geq t) \leq 3exp\left(\frac{-t}{KB}log(1+\frac{t}{2})\right) \leq 3exp\left(\frac{-tlog(2)}{KB}min(1,\frac{t}{2})\right)
\end{equation}
where the last inequality is due to the fact that for every $u>0$
we have $log(1+u)\geq log(2)min(1,u)$. Taking $\gamma=-log(2)/K$
finishes our proof.
\end{proof}

We are now ready to prove Lemma \ref{lemma: uniqnes_MC}
\begin{proof}
\label{-Decompose-any}(Lemma \ref{lemma: uniqnes_MC}) Represent any vector
$w\in R^{n_1}$ in the standard basis as $w=\sum_{i=1}^{n_1} \langle w,e_{i} \rangle e_{i}$.
Therefore $P_{U}(w)=\sum_{i=1}^{n_1} \langle P_{U}(w),e_{i} \rangle e_{i}=\sum_{i=1}^{n_1} \langle w,P_{U}(e_{i}) \rangle e_{i}$.
Recall the $r_i$ Bernoulli variables which determine if $e_i$ is included as a row of $\ar$ as in Section \ref{sec:Preliminaries-and-Notations}
and define $Y_i$ and $Z$ as in Lemma \ref{lem:bound_Z-1}. We get

\begin{align}
P_{A^{(R)^{T}}}P_{U}(w)=\sum_{i=1}^{n_1}r_{i} \langle w,P_{U}(e_{i}) \rangle e_{i}\Longrightarrow \nonumber \\
P_{U}P_{A^{(R)^{T}}}P_{U}(w)=\sum_{i=1}^{n_1}r_{i} \langle w,P_{U}(e_{i}) \rangle P_{U}(e_{i})
\end{align}
In other words the matrix $P_{U}P_{A^{(R)^{T}}}P_{U}$ is given by
\begin{equation}
P_{U}P_{A^{(R)^{T}}}P_{U} = \sum_{i=1}^{n_1} r_{i}P_{U}(e_{i})\otimes P_{U}(e_{i})
\end{equation}
 $U$ is $\mu-$incoherent, thus $max_{i\in[n_1]}||P_{U}(e_{i})||\leq\sqrt{\frac{\text{r}\mu}{n_1}}$,
hence from Lemma \ref{lem:expection_bound-1} we have for $p$ large enough:
\be
E({p}^{-1}||P_{U}P_{A^{(R)^{T}}}P_{U}-p P_{U}||_{2})<C\sqrt{\frac{log(n_1)r\mu}{p n_1}}\leq 1.
\ee

For $\beta > 1$ which satisfy the lemma's requirement, take $\lambda=\sqrt{\frac{\beta}{\gamma}}$ where $\gamma$ as in Theorem
\ref{theorem: talngrad_inequality-1}. We get that if $p>\frac{\mu log(n_1)r\beta}{n_1\gamma}$
then from Lemma \ref{lem:bound_Z-1} with probability of at least
$1-3n_{1}^{-\beta}$ we have $Z\leq C\sqrt{\frac{log(n_1)r\mu}{p n_1}}+\frac{1}{\sqrt{\gamma}}\sqrt{\frac{log(n_1)r\mu\beta}{p n_1}}$.
Taking $C_{R}=C+\frac{1}{\sqrt{\gamma}}$ finishes our proof.\end{proof}

\subsection*{Proof of Theorem \ref{theorem:Uniqness_MC}\label{proof:Uniqness_MC}}
\begin{proof}
From Lemma \ref{lemma: uniqnes_MC} and using a union bound we have that with probability
 $>1-6min(n_1,n_2)^{-\beta}$, ${\pr}^{-1}||\pr \identity r-U^{T}A^{(R)^{T}}A^{(R)}U||_{2}<1$
and ${\pc}^{-1}||\pc \identity r-V^{T}A^{(C)}A^{(C)^{T}}V||_{2}<1$. Since
the singular values of $\pr \identity r-U^{T}A^{(R)^{T}}A^{(R)}U$ are
$|\pr -\sigma_{i}(U^{T}A^{(R)^{T}}A^{(R)}U)|$ for $1\leq i\leq r$,
we have
\begin{align}
\pr -\sigma_{r}(U^{T}A^{(R)^{T}}A^{(R)}U)\leq
\sigma_{1}(\pr \identity r-U^{T}A^{(R)^{T}}A^{(R)}U)<\pr \nonumber  \\
\Rightarrow0<\sigma_{r}(U^{T}A^{(R)^{T}}A^{(R)}U)\label{eq:uniqueness_singular_values}
\end{align}

and similarly for $V^{T}A^{(C)}A^{(C)^{T}}V$. Therefore 
$rank(A^{(R)}U)=rank(V^{T}A^{(C)})=r$ and $rank(\ar X)=rank(X\ac)=r$
with probability $>1-6min(n_1,n_2)^{-\beta}$. From Lemma \ref{lem:hatX_is_X}
we get $A^{(R)}X=A^{(R)}\hat{X}\,\quad XA^{(C)}=\hat{X}A^{(C)}$ and from Lemma
\ref{lem:uniqness solution affine} we get $X=\hat{X}$ with probability $>1-6min(n_1,n_2)^{-\beta}$. \end{proof}

\subsection{Proofs for Noisy GRC Case}
\label{sec:proof_theorem_noisy_GRC}

The proof of Theorem \ref{thm:the_noisy_case} is using strong concentration results on the largest and smallest singular
values of $n\times k$ matrix with i.i.d Gaussian entries:
\begin{thm}
\cite{szarek1991condition} \label{Singular_Value_ranMatrix} Let $A\in\matrixspace nk$
be a random matrix $A\iid N(0,\frac{1}{n})$. Then, its largest and
smallest singular values obey:
\[
P\Big(\sigma_{1}(A)>1+\frac{\sqrt{k}}{\sqrt{n}}+t\Big) \text{\ensuremath{\le}}e^{-nt^{2}/2}
\]
\begin{equation}
P\Big(\sigma_{k}(A)\leq1-\frac{\sqrt{k}}{\sqrt{n}}-t\Big) \text{\ensuremath{\le}}e^{-nt^{2}/2}. \label{eq:matrix_singular_values}
\end{equation}
\end{thm}
\begin{cor}
\label{cor:psedo_inverse} Let $A\in\matrixspace nk$ be a random matrix $A\iid N(0,1)$ where $n\geq4k$, and let $\psuedoinv A$
be the Moore-Penrose pseudoinverse of $A$. Then
\begin{equation}
P\left(||\psuedoinv A||_{2}\leq\frac{6}{\sqrt{n}}\right)>1-e^{-n/18}
\end{equation}
\end{cor}
\begin{proof}
Since $\psuedoinv A$ is the pseudoinverse of $A$, $||\psuedoinv A||_{2}$=$\frac{1}{\sigma_{k}(A)}$ and
from Theorem \ref{Singular_Value_ranMatrix} we get $\sigma_{k}(A)\geq\sqrt{n}-\sqrt{k}-t\sqrt{n}$
with probability $\geq 1-e^{nt^{2}/2}$ (notice the scaling by $\sqrt{n}$ of the entries of $A$ compared to Theorem \ref{Singular_Value_ranMatrix}).
Therefore, if we take $n \geq 4k$ and $t=\frac{1}{3}$ we get
\begin{equation}
P\left(||\psuedoinv A||_{2}\leq\frac{6}{\sqrt{n}}\right)=P\left(\sigma_{k}(A)\geq\frac{\sqrt{n}}{6}\right) \geq 1-e^{-n/18}.
\end{equation}

\end{proof}

We also use the following lemma from \cite{shalev2014understanding}:
\begin{lem}
\label{lem:JL_Lemma} Let $Q$ to be a finite set of vectors in $\mathbb{R}^{n}$,
let $\delta\in(0,1)$ and $k$ be an integer such that
\begin{equation}
\epsilon \equiv \sqrt{\frac{6log(2|Q|/\delta)}{k}} \leq 3 \label{eq:k_epsilon_bound}.
\end{equation}
Let $A\in\matrixspace kn$ be a random matrix with $A\iid N(0,\frac{1}{k})$.
Then,
\begin{equation}
P\left(max_{x\in Q}\left|\frac{||Ax||^{2}}{||x||^{2}}-1\right|\leq\epsilon\right)>1-\delta. \label{eq:JL_concentration}
\end{equation}

\end{lem}
Lemma \ref{lem:JL_Lemma} is a direct result of the Johnson-Lindenstrauss
lemma \cite{dasgupta2003elementary} applied to each vector in $Q$
and using the union bound. Representing the vectors in $Q$ as a
matrix, Lemma \ref{lem:JL_Lemma} shows that $\ar,\ac$ preserve matrix Frobenius
norm with high probability - a weaker property than the RIP
which holds for \textit{any} low-rank matrix.

To prove Theorem \ref{thm:the_noisy_case}, we first represent $||X-\hat{X}||_{F}$
as a sum three parts (Lemma \ref{lem:three_parts(X-hatX)}), then
give probabilistic upper bounds to each of the parts and finally use union bound. We define $\aur=A^{(R)}\hat{U}$
and $\avc=V^{T}A^{(C)}$. From Lemma \ref{lem:unitary_normal} $\aur,\avc\iid N(0,1)$, hence $rank(\aur)=rank(\avc)=r$ with probability $1$.
We assume w.l.o.g that $\hat{X}=\hat{X}^{(R)}$ (see \alg description).
Therefore, from eq. (\ref{eq:close_formula_alg2}) we have $\hat{X}=\hat{U}(A_{\hat{U}}^{(R)^{T}}\aur)^{-1}A_{\hat{U}}^{(R)^{T}}B^{(R)}$.

We denote by $\psuedoinv{\aur}=(A_{\hat{U}}^{(R)^{T}}\aur)^{-1}A_{\hat{U}}^{(R)^{T}}$and
$\psuedoinv{\avc}=A_{V^{T}}^{(C)^{T}}(A_{V^{T}}^{(C)}A_{V^{T}}^{(C)^{T}})^{-1}$
the Moore-Penrose pseudoinverse of $\aur$ and $\avc$, respectively.
We next prove the following lemma
\begin{lem}
\label{lem:three_parts(X-hatX)} Let $\ar$ and $\ac$ be as in the
GRC model and $\zr,\zc$ be noise matrices. Let $\hat{X}$ be the output
of \alg\!. Then:

\[
||X-\hat{X}||_{F}\leq\mathbf{I+II+III}
\]

where:

\begin{align}
\mathbf{I}\equiv||(B^{(C,0)}-\svdtop{\bc}r)\psuedoinv{\avc}||_{F} \label{eq:A1} \\
\mathbf{II\equiv}||\hat{U}\psuedoinv{\aur}A^{(R)}(B^{(C,0)}-\svdtop{\bc}r)\psuedoinv{\avc}||_{F} \label{eq:A2} \\
\mathbf{III\equiv}||\hat{U}\psuedoinv{\aur}\zr||_{F}. \label{eq:A3}
\end{align}

\end{lem}

\begin{proof}
We represent $||X-\hat{X}||_{F}$ as follows
\begin{align}
||X-\hat{X}||_{F}=\nonumber \\
||X-\hat{U}(A_{\hat{U}}^{(R)^{T}}\aur)^{-1}A_{\hat{U}}^{(R)^{T}}(\ar X+\zr)||_{F}=\nonumber \\
||X-\hat{U} \psuedoinv{\aur} \ar X\: -\hat{U} \psuedoinv{\aur}  \zr||_{F}\leq\nonumber \\
||X-\hat{U} \psuedoinv{\aur} \ar X||_{F}+\mathbf{III} \label{eq:Firest_resonctruction of X}\\
\nonumber
\end{align}

where we have used the triangle inequality. We next use the following
equality
\begin{equation}
XA^{(C)}\psuedoinv{\avc}V^{T}=U\Sigma V^{T}A^{(C)}\psuedoinv{\avc}V^{T}=U\Sigma V^{T}=X
\end{equation}
to obtain:
\begin{align}
||X-\hat{U} \psuedoinv{\aur} \ar X||_{F}=\nonumber \\
||(\identity n-\hat{U}\psuedoinv{\aur}\ar)X||_{F}=\nonumber \\
||(\identity n-\hat{U}\psuedoinv{\aur}A^{(R)})XA^{(C)}\psuedoinv{\avc}V^{T}||_{F}=\nonumber \\
||(\identity n-\hat{U}\psuedoinv{\aur}A^{(R)})\bczero\psuedoinv{\avc}||_{F} \label{eq:reconstruct_the_first_part1}
\end{align}
 where the last equality is true because $V$ is orthogonal.

Since $\hat{U}$ is a basis for $span(\svdtop{\bc}r)$ there exists
a matrix $Y$ such that $\hat{U}Y=\svdtop{\bc}r$ and we get:
\begin{equation}
(\identity n-\hat{U}\psuedoinv{\aur}\ar)\svdtop{\bc}r= \svdtop{\bc}r-\hat{U}\psuedoinv{\aur}\ar\hat{U}Y= \svdtop{\bc}r-\hat{U}Y=0 . \label{eq:L_and_U}
\end{equation}

Therefore
\begin{align}
||(\identity n-\hat{U}\psuedoinv{\aur}\ar)\bczero\psuedoinv{\avc}||_{F}=\nonumber \\
||(\identity n-\hat{U}\psuedoinv{\aur}\ar)(\bczero-\svdtop{\bc}r)\psuedoinv{\avc}||_{F}\leq\nonumber \\
||(B^{(C,0)}-\svdtop{\bc}r)\psuedoinv{\avc}||_{F} + ||\hat{U}\psuedoinv{\aur}A^{(R)}(B^{(C,0)}-\svdtop{\bc}r)\psuedoinv{\avc}||_{F}=\mathbf{I+II} \label{eq:reconstruct_the_first_part}
\end{align}
Combining eq. (\ref{eq:Firest_resonctruction of X}) and eq. (\ref{eq:reconstruct_the_first_part})
gives the required result.
\end{proof}
We next bound each of the three parts in the formula of Lemma \ref{lem:three_parts(X-hatX)}.
We use the following claim:
\begin{claim}
\label{claim:BC0-BCr}$||\bczero-\svdtop{\bc}r||_{2}\leq2||\zc||_{2}$\end{claim}
\begin{proof}
We know that $||\bc-\svdtop{\bc}r||_{2}\leq||\bc-\bczero||_{2}$ since
$rank(\svdtop{\bc}r)=rank(B^{(C,0)})=r$ with probability $1$, and by definition $\svdtop{\bc}r$ is the
closest rank-$r$ matrix to $B^{(C)}$ in Frobenius norm. Therefore from
the triangle inequality
\begin{align}
||(B^{(C,0)}-\svdtop{\bc}r)||_{2}\leq ||\bc-\svdtop{\bc}r||_{2}+||\bc-\bczero||_{2}\leq \nonumber \\
2||B^{(C,0)}-B^{(C)}||_{2}=2||\zc||_{2} . \label{eq:claim_for_three_parts_bound}
\end{align}

\end{proof}
Now we are ready to prove Theorem \ref{thm:the_noisy_case}. The proof
uses the following inequalities for matrix norms for any two matrices
$A,B$: \\
\begin{align}
||AB||_{2}\leq||A||_{2}||B||_{2} \nonumber \\
||AB||_{F}\leq||A||_{F}||B||_{2} \nonumber \\
rank(A)\leqslant r \Rightarrow ||A||_{F}\leq\sqrt{r}||A||_{2}. \label{eq:three_matrix_inequalities}
\end{align}

\begin{proof}
(Theorem \ref{thm:the_noisy_case}) We prove (probabilistic) upper bounds on the three terms appearing
in Lemma \ref{lem:three_parts(X-hatX)}.
\begin{enumerate}
\item We have
\begin{equation}
rank\left((B^{(C,0)}-\svdtop{\bc}r)\psuedoinv{\avc}\right)\leqslant rank\left(\psuedoinv{\avc}\right)\leqslant r .
\end{equation}
Therefore
\begin{equation}
\mathbf{I}=||(B^{(C,0)}-\svdtop{\bc}r)\psuedoinv{\avc}||_{F}\leq
\sqrt{r}||(B^{(C,0)}-\svdtop{\bc}r)||_{2}||\psuedoinv{\avc}||_{2} \label{eq:first_Bc_BCr}
\end{equation}
Since $\avc\iid N(0,1)$, from Corollary $\ref{cor:psedo_inverse}$ we get
$P\Big(||\psuedoinv{\avc}||_{2}\leq \frac{6}{\sqrt{k}}\Big) \geq 1-e^{-k/18}$ for $k\geq4r$,
hence with probability $\geq 1-e^{-k/18}$,
\begin{equation}
\mathbf{I}\leq6\sqrt{\frac{r}{k}}||(\bczero-\svdtop{\bc}r)||_{2}. \label{eq:BcBcrSQRT}
\end{equation}
From Claim \ref{claim:BC0-BCr} and eq. (\ref{eq:A1}) we get a bound on $\mathbf{I}$ for some absolute constants $C_{1},c_{1}$:
\begin{equation}
P\Big(\mathbf{I}\leq C_{1}\sqrt{\frac{r}{k}}||\zc||_{2} \Big) > 1-e^{-c_{1}k} . \label{eq:bound_A1}
\end{equation}

\item $\hat{U}$ is orthogonal and can be omitted from $\mathbf{II}$ without
changing the norm. Applying the second inequality in eq. (\ref{eq:three_matrix_inequalities}) twice, we get the
inequality:
\begin{align}
\mathbf{II=}||\hat{U}\psuedoinv{\aur}\ar(\bczero-\svdtop{\bc}r)\psuedoinv{\avc}||_{F}\leq \nonumber \nonumber \\
||\psuedoinv{\aur}||_{2}||\ar(\bczero-\svdtop{\bc}r)||_{F} ||\psuedoinv{\avc}||_{2}. \label{eq:temp_equation2}
\end{align}

From Corollary \ref{cor:psedo_inverse} we know that for $k>4r$ we have $||\psuedoinv{\aur}||_{2}\leq\frac{6}{\sqrt{k}}$
and $||\psuedoinv{\avc}||_{2}\leq\frac{6}{\sqrt{k}}$, each with probability $>1-e^{-k/18}$. Therefore,
\begin{equation}
P \Big( \mathbf{II}\leq\frac{36}{k}||\ar(B^{(C,0)}-\svdtop{\bc}r)||_{F} \Big) > 1-2e^{-k/18}. \label{eq:Second(Bc-Br)}
\end{equation}
$A^{(R)}$ and $\bczero-\svdtop{\bc}r$ are independent and $rank(B^{(C,0)}-\svdtop{\bc}r)\leq2r$.
Therefore we can apply Lemma \ref{lem:JL_Lemma} with $k$ such that
$\frac{k}{6}>log(2k)+\frac{k}{18}$ (this holds for $k\geq40$) to
get with probability $>1-2e^{-k/18}$:
\begin{align}
\mathbf{II}\leq\frac{36}{k}||\ar(\bczero-\svdtop{\bc}r)||_{F}\leq \nonumber \\
\frac{36\sqrt{2k}}{k}||(B^{(C,0)}-\svdtop{\bc}r)||_{F}\leq
36\sqrt{4\frac{r}{k}}||(B^{(C,0)}-\svdtop{\bc}r)||_{2}. \label{eq:secondBcBcrSQRT}
\end{align}
From eq. (\ref{eq:Second(Bc-Br)}) and (\ref{eq:secondBcBcrSQRT})
together with Claim \ref{claim:BC0-BCr} we have constants $C_{2}$
and $c_{2}$ such that,
\begin{equation}
P \Big (\mathbf{II}\leq C_{2}||\zc||_{2} \Big) > 1-3e^{-c_2k} . \label{eq:bound_A2}
\end{equation}

\item $rank(\psuedoinv{\aur})\leq r$ and from Corollary \ref{cor:psedo_inverse} we get $P\Big(||\psuedoinv{\aur}||_{2}\leq\frac{6}{\sqrt{k}}\Big)>1-e^{-k/18}$
for $k>4r$. Therefore, with probability $>1-e^{-k/18}$:
\begin{align}
\mathbf{III=}||\hat{U}\psuedoinv{\aur}\zr||_{F}=
||\psuedoinv{\aur}\zr||_{F}\leq \nonumber \\
\sqrt{r}||\psuedoinv{\aur}\zr||_{2}\leq
\sqrt{r}||\psuedoinv{\aur}||_{2}||\zr||_{2}\leq\frac{6\sqrt{r}}{\sqrt{k}}||\zr||_{2} . \label{eq:Z(R)_bound}
\end{align}

\end{enumerate}
Hence we have constants $C_{3}$ and $c_{3}$ such that,
$>1-e^{-c_{3}k}.$
\begin{equation}
P\Big(\mathbf{III}\leq C_{3}||\zr||_{2} \Big) > 1-e^{-c_{3}k}. \label{eq:bound_A3}
\end{equation}

Combining equations (\ref{eq:bound_A1},\ref{eq:bound_A2},\ref{eq:bound_A3})
with Lemma \ref{lem:three_parts(X-hatX)} and taking the union bound while setting $\ccol=C_{1}+C_{2}$,
$\crow=C_{3}$ with $c=min(c_{1},c_{2},c_{3})$ concludes our proof.
\end{proof}


\newpage

\subsection{Simulations for Large Values of $n$} \label{sec:appendix_simulations}

We varied $n$ between $10$ and $1000$, with
results averaged over $100$ different matrices of rank $3$ at each
point, and tried to recover them using $k=20$ row and column measurements.
Measurement matrices were $A^{(R)},A^{(C)}\iid \frac{1}{n}$ to allow similar norms
for each measurement vector for different values of $n$.
Recovery performance was insensitive to $n$. if we take $A^{(R)},A^{(C)}\iid N(0,1)$ instead
of $N(0,\frac{1}{n})$, the scaling of our results is in agreement with Theorem \ref{thm:the_noisy_case}.
\begin{figure}[H]
\begin{centering}
\includegraphics[width=12cm,height=8cm]{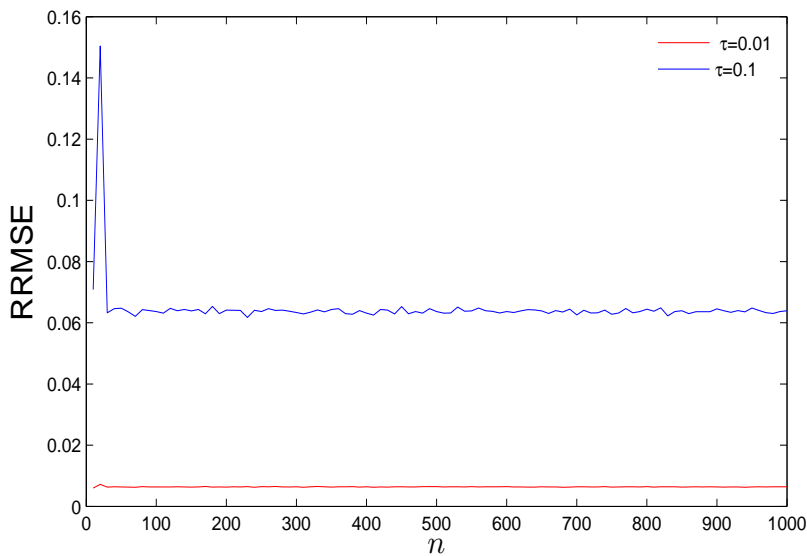}
\par\end{centering}

\protect\caption{Reconstruction error for $n\times n$ matrix where $n$ is varied
between $10$ and $1000,\thinspace k=20$ and $r=3$ and two different
noise levels: $\tau=0.1$ (blue) and $\tau=0.01$ (red). Each point
represents average performance over $100$ random matrices. \label{fig:scaling_n_k}}

\end{figure}

Next, we take $n,k,r\rightarrow\infty$ while the ratios $\frac{n}{k}=5$
and $\frac{k}{r}=4$ are kept constant, and compute the relative error
for different noise level. Again, the relative error converges rapidly
to constant, independent of $n,k,r$ .
\begin{figure}[H]
\begin{centering}
\includegraphics[width=12cm,height=8cm]{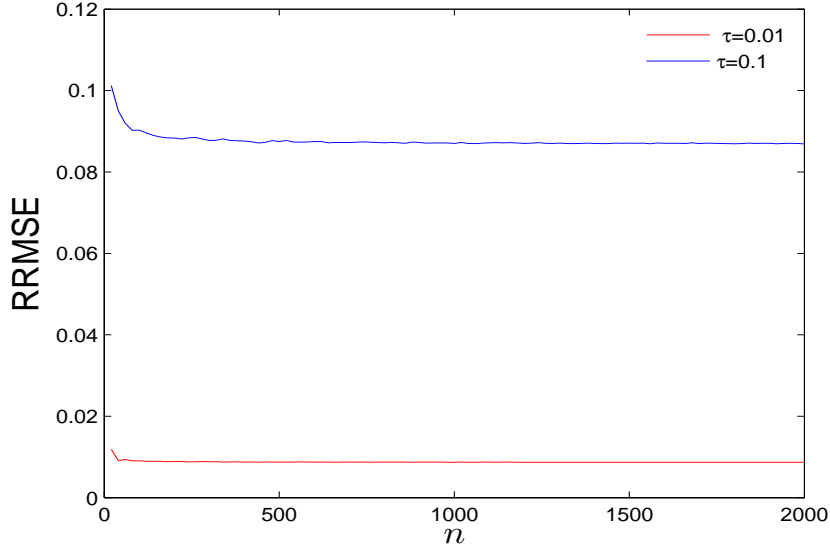}
\par\end{centering}

\protect\caption{Reconstruction error for $n\times n$ matrix $X$ with rank $r$ varying from $1$ to $50$ and with $n=20r, k=4r$.
Two different noise level are shown: $\tau=0.1$ (blue) and $\tau=0.01$ (red). Each point represents
average performance over $100$ random matrices. \label{fig:scaling_r_n_k}}

\end{figure}

\subsection{Low Rank matrix Approximation} \label{sec:Low Rank matrix}
We bring here the one pass algorithm to approximate $X$ from \cite{halko2011finding} for the convenience of the reader.
The output of this algorithm isn't low rank if $k>r$. This algorithm is different from $\algp$ and its purpose is to
approximate a (possibly full rank) matrix by low rank matrix. We adjusted Algorithm \ref{alg:-Halko} to our purpose with some changes.
First, we estimate the rank of $X$ using the elbow method from Section \ref{sec:estimate_unknown_rank} and instead of calculating the QR decomposition
of $B^{(C)}$ and $B^{(R)^{T}}$ we find their $\hat{r}$ largest singular vectors. Furthermore, we repeat part two in algorithm
\ref{alg:-Halko} while replacing the roles of columns and rows as in \alg.
This variation gives our modified algorithm $\algp$ as described in Section \ref{sec:Low_rank_aprox}.
\begin{algorithm}
Input: $A^{(R)},A^{(C)},B^{(R)},B^{(C)}$
\begin{enumerate}
\item compute $Q^{(C)}R^{(C)}$ the QR decomposition of $B^{(C)}$,
and $Q^{(R)}R^{(R)}$ the QR decomposition of $B^{(R)^{T}}$
\item Find the least-squares solution $Y=argmin_{C}||Q^{(C)}B^{(C)}-CQ^{(R)^{T}}B^{(R)^{T}}||_{F}$.
\item Return the estimate $\hat{X}=Q^{(C)}YQ^{(R)^{T}}$.

\end{enumerate}
\protect\caption{\label{alg:-Halko}}

\end{algorithm}

We compared our \alg to $\algp$ which is presented in Section \ref{sec:Low_rank_aprox}.
We took $X \in \lowrank{1000}{1000}{10}$ and $\sigma =1$. We tried to recover $X$ in the GRC model with $k=12$ for $100$ different matrices.
For each matrix, we compared the $\RRMSE$ obtained for the outputs of \alg and $\algp$.
The $\RRMSE$ for $\algp$ was lower than the $\RRMSE$ for \alg in most cases but the differences were very small and negligible.
 \begin{figure}[H]
\begin{centering}
\includegraphics[width=12cm,height=8cm]{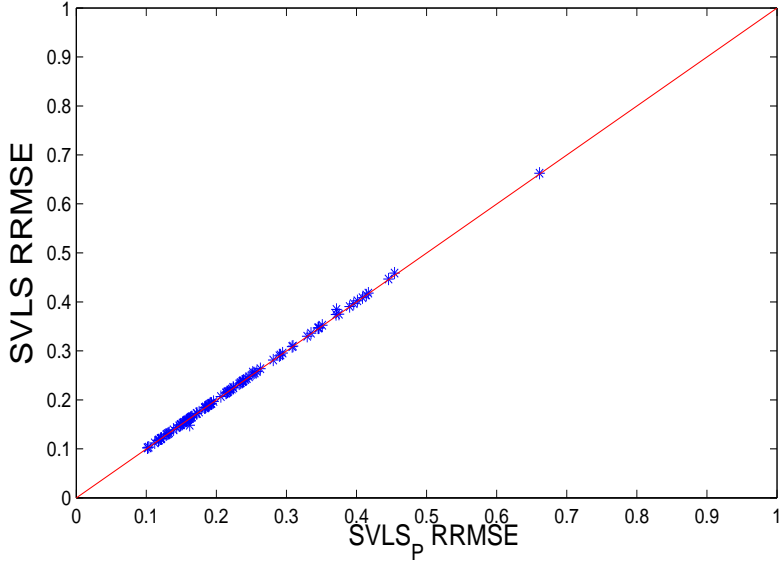}
\par\end{centering}

\protect\caption{We recover a matrix $X$ from $24000$ measurements as in the GRC model $100$ times. Figure shows average $\RRMSE$ over $100$
simulations for \alg ($Y$ axis) and $\algp$ ($X$ axis). The red linear line $Y=X$ was drawn for comparing those two algorithm,
every dot that under the red line is a simulation that SVLS was better than $\algp$ and every dot above the line tells the opposite}

\end{figure}

\subsection{Rank Estimation} \label{sec:Rank estimation}
We test the elbow method for estimating the rank of $X$ (see eq. (\ref{eq:rank_estimation})).
We take a matrix $X$ of size $400\times 400$ and different ranks. 
We add Gaussian noise with $\sigma = 0.25$ while the measurements are sampled as in the RCMC model.
For each number of measurements we sampled $100$ matrices and took the average estimated rank.
We compute the estimator $\hat{r}$ for different values of $d$, the number of measurements.
We compare our method to the rank estimation which appears in \algo \cite{keshavan2009matrix} 
for the standard MC problem.
Our simulation results, shown in Figure \ref{fig:rank_estimation}, indicate that the RCMC model 
with the elbow method is a much better design for rank estimation of $X$.

\begin{figure}[H]
\begin{centering}
\includegraphics[scale=0.7]{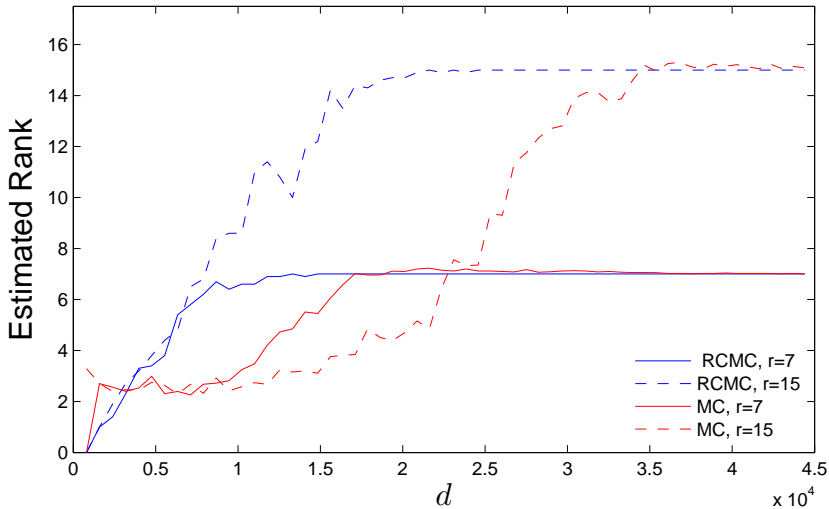}
\par\end{centering}

\protect\caption{Estimation of $rank(X)$ vs. $d$, the number of measurements, $d=k(2n-k)$ where $k$ is the number of columns in $\bc$ and number of rows in $\br$.
For each $d$ we sampled $100$ different matrices. Estimation was performed by the elbow method for RCMC model, as in eq. (\ref{eq:rank_estimation}) in the main text,
and for the MC model we used the method described in \cite{keshavan2009matrix}. RCMC recovers the correct rank with smaller number of measurements. \label{fig:rank_estimation}}
\end{figure}

\subsection{Test Error} \label{sec:test error}
In matrix completion with MC and RCMC ensembles the $\RRMSE$ loss function measures the loss 
on both the observed and unobserved entries. This loss may be too optimistic when
considering our prediction error only on unobserved entries.
Thus, instead of including all measurements in calculation of the $\RRMSE$ we compute a 
different measure of prediction error, given by the $\RRMSE$ only on the unobserved entries.
For each single-entry measurements operator $\affine$ define $E(\affine)$ the set of measured 
entries and $\bar{E}$ it's complement, i.e. the set of unmeasured entries $(i,j)\in [n_1]\times [n_2]$.
We define $X^{\bar{E}}$ to be a matrix such that $X^{\bar{E}}_{ij}=X_{ij}$ if $(i,j)\in \bar{E}$ and $0$ otherwise.
Instead of $\RRMSE(X,\hat{X})$ we now calculate $\RRMSE(X^{\bar{E}},\hat{X}^{\bar{E}})$. 
This quantity measures our reconstruction only on the unseen matrix entries $X_{ij}$,
and is thus not influenced by overfitting. In Table \ref{tab:accuracy_overffiting} we 
performed exactly the same simulation as in Table \ref{tab:accuracy-and-time} but with
$\RRMSE(X^{\bar{E}},\hat{X}^{\bar{E}})$. The results of \algo\!, SVT and SVLS stay similar to the results in Table \ref{tab:accuracy-and-time} and our $\RRMSE$ loss function does not show overfitting.

\begin{table}[H]
\centering
\small
\begin{tabular}{|c|c|c|c|c|c|}
  \hline
  \snr & $d$ & $r$ & \alg & \algo & SVT \\ \hline
  $10^{-2}\!\!\!$ & $120156$ & $10$ & $0.006\quad(0.006) $ & $0.004\quad(0.004)$ & $0.0074\quad(0.0073)$ \\
  $10^{-1}\!\!\!$ & $120156$ & $10$ & $0.065\quad(0.064)$ & $0.045\quad(0.044)$ & $0.051\quad(0.05)$ \\
  $1\!\!\!$ & $120156$ & $10$ & $0.619\quad(0.612)$ & $0.49\quad(0.49)$ & $0.52\quad(0.51)$ \\

  \hline
\end{tabular}
\caption{$\protect\RRMSE$ only on the unknown measurements. for \alg applied to RCMC, and
\algo and SVT applied to the standard MC. Results represent average of
$5$ different random matrices. The results in the parentheses are the standard $\RRMSE$ in Table \ref{tab:accuracy-and-time}. \label{tab:accuracy_overffiting}}

\end{table}

\end{document}